\documentclass{article}

\usepackage[nonatbib, final]{neurips_2020}

\usepackage[utf8]{inputenc} 
\usepackage[T1]{fontenc}    
\usepackage{hyperref}       
\usepackage{url}            
\usepackage{booktabs}       
\usepackage{amsfonts}       
\usepackage{nicefrac}       
\usepackage{microtype}      

\usepackage{graphicx}
\usepackage{caption}
\usepackage{subcaption}

\usepackage{booktabs}
\usepackage{amsthm, amsmath, amssymb}
\usepackage{bm}
\usepackage{bbm}

\usepackage{color}

\usepackage[dvipsnames,svgnames,x11names,hyperref]{xcolor}

\usepackage{hyperref}
\hypersetup{colorlinks,breaklinks,      citecolor=ForestGreen,urlcolor=NavyBlue, linkcolor=NavyBlue}

\usepackage[nameinlink]{cleveref}
\usepackage{enumitem}

\allowdisplaybreaks

\makeatletter
\newcommand*\rel@kern[1]{\kern#1\dimexpr\macc@kerna}
\newcommand*\widebar[1]{%
  \begingroup
  \def\mathaccent##1##2{%
    \rel@kern{0.8}%
    \overline{\rel@kern{-0.8}\macc@nucleus\rel@kern{0.2}}%
    \rel@kern{-0.2}%
  }%
  \macc@depth\@ne
  \let\math@bgroup\@empty \let\math@egroup\macc@set@skewchar
  \mathsurround\z@ \frozen@everymath{\mathgroup\macc@group\relax}%
  \macc@set@skewchar\relax
  \let\mathaccentV\macc@nested@a
  \macc@nested@a\relax111{#1}%
  \endgroup
}
\makeatother

\newcommand{\eps}{\epsilon}

\newcommand{\norm}[1]{\lVert#1\rVert}

\usepackage{algorithm}
\usepackage{algorithmic}

\newtheorem{theorem}{Theorem}[section]
\newtheorem{lemma}[theorem]{Lemma}
\newtheorem*{lemma*}{Lemma}
\newtheorem{claim}[theorem]{Claim}
\newtheorem{proposition}[theorem]{Proposition}
\newtheorem{definition}{Definition}
\newtheorem{assumption}{Assumption}

\newcommand{\E}{\mathbb{E}}
\newcommand{\R}{\mathbb{R}}
\newcommand{\inn}[2]{\langle #1,#2\rangle}
\newcommand{\wh}[1]{\widehat{#1}}

\DeclareMathOperator*{\argmin}{arg\,min}

\graphicspath{{./figures/}}
\title{Robust Compressed Sensing using Generative Models}

\author{%
 Ajil Jalal \thanks{Link to our code: \url{https://github.com/ajiljalal/csgm-robust-neurips}} \\
 ECE, UT Austin \\
 \texttt{ajiljalal@utexas.edu} \\
 \And
 Liu Liu \\
 ECE, UT Austin \\
 \texttt{liuliu@utexas.edu} \\
 \AND
 Alexandros G.~Dimakis \\
 ECE, UT Austin \\
 \texttt{ dimakis@austin.utexas.edu} \\
 \And
 Constantine Caramanis \\
 ECE, UT Austin \\
 \texttt{constantine@utexas.edu} \\
}

\begin{document}

\maketitle
\begin{abstract}
The goal of compressed sensing is to estimate a high dimensional vector from an underdetermined system of noisy linear equations. In analogy to classical compressed sensing, here we assume a generative model as a prior, that is, we assume the 
vector is represented by a 
deep generative model $G: \R^k \rightarrow \R^n$.
Classical recovery approaches such as empirical risk minimization (ERM) are guaranteed to succeed when the  measurement matrix is sub-Gaussian.
However, when the measurement matrix and measurements are heavy-tailed or have outliers, recovery may fail dramatically.
In this paper we propose an algorithm inspired by the Median-of-Means (MOM). Our algorithm guarantees recovery for heavy-tailed data, even in the presence of outliers. Theoretically, our results show our novel MOM-based algorithm enjoys the same sample complexity guarantees as ERM under sub-Gaussian assumptions.
Our experiments validate both aspects of our claims: other algorithms are indeed fragile and fail under heavy-tailed and/or corrupted data, while our approach exhibits the predicted robustness.
\end{abstract}

\section{Introduction}
Compressive or compressed sensing is the problem of reconstructing an unknown vector $x^* \in \mathbb{R}^n$ after observing $m<n $ linear measurements of its entries, possibly with added noise: $y = Ax^* + \eta,$
where $A \in \mathbb{R}^{m \times n}$ is called the measurement matrix and $\eta \in \mathbb{R}^m$ is noise. Even without noise, this is an underdetermined system of linear equations, so recovery is impossible without a structural assumption on the unknown vector $x^*$. The vast literature~\cite{tibshirani1996regression,hastie2015statistical,negahban2012unified,bach2012optimization,candes2006robust,donoho2006compressed,agarwal2012,tropp2007signal,baraniuk2007compressive} on this subject typically assumes that the unknown vector is ``natural,'' or ``simple,'' in some application-dependent way.

Compressed sensing has been studied on a wide variety of structures  such as sparse 
vectors~\cite{candes2006stable}, trees~\cite{chen2012compressive}, graphs~\cite{xu2011compressive}, manifolds~\cite{chen2010compressive,xu2008compressed} or deep generative models~\cite{bora2017compressed}. In this paper, we
concentrate on deep generative
models, which were explored by~\cite{bora2017compressed} 
as priors for 
sample-efficient reconstruction. 
Theoretical results in~\cite{bora2017compressed} showed that if $x^*$ lies close to the
range of a generative model $G: \R^k \to \R^n$ with $d-$layers, a variant of ERM can 
recover $x^*$ with $m=O(kd\log n) $ measurements. 
Empirically,~\cite{bora2017compressed} shows that 
generative models require $5-10\times$ fewer measurements
to obtain the same reconstruction accuracy as Lasso.
This impressive empirical performance has motivated 
significant recent research to better understand the behaviour and theoretical limits of compressed sensing using generative priors~\cite{hand2017global, kamath2019lower, liu2019information}

A key technical condition for recovery is the 
Set Restricted Eigenvalue Condition (S-REC)~\cite{bora2017compressed},
which is a generalization of the Restricted Eigenvalue Condition
~\cite{bickel2009simultaneous, candes2008restricted} in sparse recovery.
This condition is satisfied if $A$ is a sub-Gaussian matrix and the measurements satisfy $y = Ax^* + \eta$.
This leads to the question: can the conditions on $A$ be weakened, 
and can we allow for outliers in $y$ and $A$? 
This has significance in applications such as MRI and astronomical 
imaging, where data is often very noisy and requires significant
pruning/cleansing.

As we show in this paper, the analysis and algorithm proposed 
by~\cite{bora2017compressed} are quite fragile in the presence of 
heavy-tailed noise or corruptions in the measurements.
In the statistics literature, it is well known that algorithms
such as empirical risk minimization (ERM) and its variants are
not robust to even a \emph{single} outlier. Since the algorithm
in~\cite{bora2017compressed} is a variant of ERM, it is susceptible to the same
failures in the presence of heavy-tails and outliers. 
Indeed, as we show empirically in Section \ref{sec:experiments},
precisely this occurs. 

Importantly, recovery failure in the setting of
\cite{bora2017compressed} (which is also the focus of this paper)
can be pernicious, precisely because generative models (by
design) output images in their range space, and for well-designed
models, these have high perceptual quality. In contrast, when a
classical algorithm like LASSO~\cite{tibshirani1996regression} fails, the typical failure mode is
the output of a non-sparse vector. Thus in the context of
generative models, resilience to outliers and heavy-tails is
especially critical. This motivates the need for algorithms that
do not require strong assumptions on the measurements. 

In this paper, we propose an algorithm for compressed sensing
using generative models, which is robust to heavy-tailed
distributions and arbitrary outliers. We study its theoretical
recovery guarantees as well as empirical performance, and show
that it succeeds in scenarios where other existing recovery
procedures fail, without additional cost in sample complexity or computation.

\subsection{Contributions}

We propose a new reconstruction algorithm in place of ERM. 
Our algorithm uses a Median-of-Means (MOM) loss to 
provide robustness to heavy-tails and arbitrary
corruptions. As S-REC may no longer hold, we
necessarily use a different analytical approach.
We prove recovery results and sample complexity guarantees for this setting even though previous assumptions such as the S-REC~\cite{bora2017compressed} condition do not hold. 
Specifically, our main contributions are as follows.
\begin{itemize}[leftmargin = *]
    \item (Algorithm) We consider robust compressed sensing for generative models where (i) a constant fraction of 
the measurements and measurement matrix are arbitrarily (perhaps maliciously) corrupted and (ii) the random ensemble only satisfies a weak moment assumption. 

We propose a novel algorithm to replace ERM. Our algorithm uses a median-of-means (MOM) tournament~\cite{lugosi2016risk,lecue2017robust} i.e., a min-max optimization framework for robust reconstruction. Each iteration of our MOM-based algorithm comes at essentially no additional computational cost compared to an iteration of standard ERM. Moreover, as our code shows, it is straightforward to implement. 

    \item (Analysis and Guarantees) We analyze the recovery guarantee and outlier-robustness of our algorithm when the generative model
    is a $d$-layer neural network using ReLU activations. Specifically, in the presence of a constant fraction of outliers in $y$ and $A$, we  achieve $\|G(\wh{z}) - G(z^*)\|^2 \leq O(\sigma^2 + \tau)$ with sample size $m =O(kd\log n)$, where $\sigma^2$ is the variance of the heavy-tailed noise, and $ \tau$ is the optimization accuracy. Using different analytical tools (necessarily, since we do not assume sub-Gaussianity), we show our algorithm, even under heavy-tails and corruptions, has the same sample complexity as the previous literature has achieved under much stronger sub-Gaussian assumptions. En route to our result, we also prove an interesting result for ERM: by avoiding the S-REC-based analysis, we show that the standard ERM algorithm does in fact succeed in the presence of a heavy-tailed measurement matrix, thereby strengthening the best-known recovery guarantees from \cite{bora2017compressed}. This does not extend (as our empirical results demonstrate) to the setting of outliers, or of heavy-tailed measurement noise. For these settings, our new algorithm is required.
    \item (Empirical Support)
    We empirically validate the effectiveness of our robust recovery algorithm on MNIST and CelebA-HQ. Our results demonstrate that (as our theory predicts) our algorithm succeeds in the presence of heavy-tailed noise, heavy-tailed measurements, and also in the presence of arbitrary outliers. At the same time our experiments confirm that ERM can fail, and in fact fails dramatically: through an experiment on the CelebA-HQ data set, we demonstrate that the ERM recovery approach \cite{bora2017compressed}, as well as other natural approaches including $\ell_1$ loss minimization and trimmed loss minimization \cite{shensujay2018learning}, can recover images that have little resemblance to the original.
\end{itemize}

\subsection{Related work}

Compressed sensing with outliers or heavy-tails has a long history.  To deal with outliers only in $y$,
classical techniques replace the ERM with a robust loss function such as $\ell_1$ loss or Huber loss \cite{li2013compressed,nguyen2013exact,loh2017statistical, dalalyan2019outlier}, and obtain the optimal statistical rates. 
Much less is known for outliers in $y$ and $A$ for robust compressed sensing.  
Recent progress on robust sparse regression \cite{chen2013robust, du2017computationally,chen2017distributed,sever2018,ravikumar2018robust,liu2018RSGE, liu2019RDC,shensujay2018learning} can handle outliers in $y$ and $A$, but their techniques
cannot be directly extended to 
arbitrary generative models $G$.
Another line of research \cite{hsu2016loss,minsker2015geometric,
lugosi2016risk,lecue2017robust} considers compressed sensing where 
the measurement matrix $A$ and $y$ have 
heavy-tailed distributions.
Their techniques leverage variants of Median-of-Means (MOM) estimators on the loss function under weak moment assumptions instead of sub-Gaussianity, which generalize the classical MOM mean estimator in one dimension
\cite{MOM_nemirovsky1983problem,MOM_jerrum1986random,MOM_alon1999space,minsker2015geometric}.

\cite{wei2019statistical} deals with compressed sensing of generative models when the measurements and the responses are non-Gaussian. However, the distribution model in \cite{wei2019statistical} requires more stringent conditions compared to the weak moment assumption as will be specified in \Cref{def:heavy-tailed}, and their algorithm cannot tolerate arbitrary corruptions.
\cite{yi2018outlier} consider $\ell_1$-minimization for outlier detection using generative models, assuming the outliers in $y$ are sparse.

Generative priors have shown great promise in compressed sensing and 
other inverse problems, starting with~\cite{bora2017compressed}, who generalized the theoretical framework of compressive sensing
and restricted eigenvalue conditions~\cite{tibshirani1996regression,donoho2006compressed,bickel2009simultaneous,candes2008restricted, hegde2008random,baraniuk2009random,baraniuk2010model,eldar2009robust} 
for signals lying on the range of a deep generative model~\cite{goodfellow2014generative, kingma2013auto}. 
 Results in~\cite{kamath2019lower,liu2019information,jalali2019solving} 
established that the sample complexities in~\cite{bora2017compressed} are order optimal.
The approach in~\cite{bora2017compressed} has been generalized to tackle different 
inverse problems~\cite{hand2018phase, aubin2019exact, asim2018blind, mosser2020stochastic,asim2018solving,qiu2019robust,aubin2019exact,liu2020sample,asim2019invertible,jagatap2019phase,hand2019global,anirudh2019mimicgan}.
Alternate algorithms for reconstruction include ~\cite{bora2018ambientgan,dhar2018modeling, kabkab2018task, fletcher2018inference, fletcher2018plug, song2019surfing, mardani2018deep, dhar2018modeling, pandit2019inference,heckel2018deep,heckel2020compressive}.
The complexity of optimization algorithms using generative models 
have been analyzed in~\cite{gomez2019fast, hegde2018algorithmic, lei2019inverting,hand2017global}.
See~\cite{ongie2020deep} for a more detailed survey on deep learning techniques for compressed sensing.
A related line of work has explored learning-based approaches to tackle classical problems
in algorithms and signal processing~\cite{aamand2019learned,indyk2019learning, metzler2017learned, hsu2018learning}.

\section{Notation}
For functions $f(n)$ and $g(n),$ we write $f(n)\lesssim g(n)$ to denote that there exists a universal constant $c_1 > 0$ such that $f(n) \leq c_1 g(n)$.
Similarly, we write $f(n) \gtrsim g(n)$ to denote that there exists a universal constant $c_2 > 0$ such that $f(n) \geq c_2 g(n)$.
We write $f(n) = O(g(n))$ to imply that there exists a positive constant $c_3$ and a natural number $n_0$ such that for all $n \geq n_0,$ we have $|f(n)| \leq c_3 g(n)$.
Similarly, we write $f(n) = \Omega(g(n))$ to imply that there exists a positive constant $c_4$ and a natural number $n_1$ such that for all $n \geq n_1,$ we have $|f(n)| \geq c_4 g(n)$.

\section{Problem formulation}
Let $x^* = G(z^*)\in\R^n$ be the
fixed vector of interest. 
The deep generative model
$G:\R^k\rightarrow \R^n$ ($k\ll n$)
maps from a low dimensional latent space to a higher dimensional space. 
In this paper, $G$ is a feedforward neural network with ReLU activations and $d$ layers.

Our definition of heavy-tailed samples assumes that the measurement matrix  $A$ only 
has bounded fourth moment. Our corruption model is Huber's $\epsilon$-contamination 
model \cite{huber1964robust}. 
This model allows corruption in 
the measurement matrix $A$ and measurements $y$.
Precisely, these are:
\begin{definition}[Heavy-tailed samples]
\label{def:heavy-tailed}
We say that a random vector $a$ is heavy-tailed if for a universal constant
$C>0$, the $4^{th}$ moment of  $a$ satisfies
\[
\left(\E\left[\inn{a}{u}^{4} \right]\right)^\frac{1}{4} \leq
C\left(\E\left[\inn{a}{u}^2 \right]\right)^\frac{1}{2}, \qquad \forall  u
\in\R^n.
\]
For all $\delta >0$, the $(4+\delta)^{th}$ moment of $a$ need not exist, and we
make no assumptions on them.
\end{definition}

\begin{definition}[$\eps$-corrupted samples]
\label{def:corruption}
We say that a collection of samples  $ \{y_i, a_i\}$ is
$\eps$-corrupted if they are  i.i.d.\ observations drawn from the
mixture
$$ \{y_i, a_i\} \sim (1-\eps) P + \eps Q,$$
 where $P$ is the uncorrupted distribution, $Q$ is an arbitrary
 distribution.
\end{definition}

Thus we assume that samples $\{y_i, a_i\}_{i=1}^m$ are generated from
$(1-\eps) P + \eps Q$, where $Q$ is an adversary, and $P$ satisfies
the following:
\begin{assumption}
\label{assm:assumption}
Samples $(y_i, a_i)\sim P$ satisfy $y_i = a_i^\top G(z^*) + \eta_i,$
where the random vector $a_i$ is isotropic and heavy-tailed as in
\Cref{def:corruption}, and the noise term $\eta_i$ is independent of
$a_i$, i.i.d.\ with zero mean and bounded variance $\sigma^2$.
\end{assumption}

\section{Our Algorithm}
$\norm{\cdot}$ refers to $\ell_2$ unless specified otherwise.
The procedure proposed by~\cite{bora2017compressed} finds a reconstruction $\hat{x} = G(\hat{z})$, where $\hat{z}$ solves:
\[
\wh{z} := \argmin_{z\in\R^k}\norm{AG(z) - y}^2.
\]
This is essentially an ERM-based approach. As is well known from the classical statistics literature, ERM's success relies on strong concentration properties, guaranteed, e.g., if the data are all sub-Gaussian. ERM may fail, however, in the presence of corruption or heavy-tails. Indeed, our experiments demonstrate that in the presence of outliers in $y$ or $A$, or heavy-tailed noise in $y$, \cite{bora2017compressed} fails to recover  $G(z^*)$.

{\bf Remark} {\em Unlike typical problems in $M$-estimation and high dimensional statistics, the optimization problem that defines the recovery procedure here is non-convex, and thus in the worst case, computationally intractable. Interestingly, despite non-convexity, as demonstrated in \cite{bora2017compressed}, (some appropriate version of) gradient descent
is empirically very successful. In this paper, we take this as a computational primitive, thus sidestepping the challenge of proving whether a gradient-descent based method can efficiently provide guaranteed inversion of a generative model. Our theoretical guarantees are therefore statistical but our experiments show empirically excellent performance.}

\subsection{MOM objective}

It is well known that the median of means estimator achieves nearly sub-Gaussian concentration for one dimensional mean estimation of variance bounded random variables \cite{MOM_nemirovsky1983problem,MOM_jerrum1986random,MOM_alon1999space}. 
Inspired by the median-of-means algorithm, we propose the following algorithm to handle heavy-tails and outliers in $y$ and $A$. We partition the set $\left[m\right]$ into $M$ disjoint batches $\{B_1, B_2, \cdots, B_M\}$ such that each batch has cardinality $b=\frac{m}{M}$. Without loss of generality, we assume that $M$ exactly divides $m$, so that $b$ is an integer.
For the $j^{th}$ batch $B_j$, define the function
\begin{align}\label{equ:loss}
    \ell_j(z) := \frac{1}{b} \| A_{B_j}G(z) - y_{B_j}\|^2,
\end{align}
where $A_{B_j} \in\R^{b\times n}$ denotes the submatrix of $A$ corresponding to the rows in batch $B_j$. Similarly, $y_{B_j}\in\R^{b}$ denotes the entries of $y$ corresponding to the batch $B_j.$
Our workhorse is a novel variant of median-of-means (MOM) tournament procedure \cite{lugosi2016risk,lecue2017robust}
using the loss function \cref{equ:loss}:
\begin{align}
    \widehat{z} = \arg \min_{z\in\R^k}  \max_{z'\in\R^k} 
    \underset{1\leq j \leq M}{\mathrm{median}} (\ell_j(z) - \ell_j(z')).\label{equ:MOM_minmax}
\end{align}
We do not assume that the minimizer is unique, since we only require a reconstruction $G(\wh{z})$ which is close to $G(z^*)$. Any value of $z$ in the set of minimizers will suffice.
The intuition behind this aggregation of batches is that if the inner player $z'$ chooses 
a point close to $z^*,$ then the outer player $z$ must also choose a 
point close to $z^*$ in order to minimize the objective.
Once this happens, there is no better option for $z'$.
Hence a neighborhood around $z^*$ is almost an equilibrium, and in fact
there can be no neighborhood far from $z^*$ with such an equilibrium.

\paragraph{Computational considerations.}
The objective function \cref{equ:MOM_minmax} is not convex and we use  \Cref{alg:MOM_GAN} as a heuristic to solve \cref{equ:MOM_minmax}. 
In \Cref{sec:experiments}, we empirically observe that gradient-based methods are able to minimize this objective and have  good convergence properties. Our main theorem guarantees that a small value of the objective implies a good reconstruction and hence we can certify reconstruction quality
using the obtained final value of the objective. 

\begin{algorithm}[t]
\footnotesize
\begin{algorithmic}[1]
\STATE \textbf{Input:} Data samples $\{y_j, a_j\}_{j=1}^m$.
\STATE \textbf{Output:} $G(\widehat{z})$.
\STATE \textbf{Parameters:} Number of batches $M$.\\
{\kern2pt \hrule \kern2pt}
\STATE Initialize $z$ and $z'$.
\FOR {$t=0$ to $T-1$,}
\STATE For each batch $j\in [M]$, calculate
 $\frac{1}{|B_j|} (\ell_j(z) - \ell_j(z'))$ by \cref{equ:loss}.
\STATE Pick the batch with the median loss
$\underset{1\leq j \leq M}{\mathrm{median}} (\ell_j(z) - \ell_j(z'))$, and evaluate the gradient for $z$ and $z'$ using backpropagation on that batch. \\
 (i) perform gradient descent for $z$; \\
 (ii) perform gradient ascent for $z'$. 
\ENDFOR
\STATE Output the  $G(\widehat{z}) = G({z})$.

\end{algorithmic}
\caption{\footnotesize Robust compressed sensing of generative models}
\label{alg:MOM_GAN}
\end{algorithm}

\section{Theoretical results}
We begin with a brief review of the Restricted Eigenvalue Condition in standard 
compressed sensing and show that S-REC is satisfied by heavy-tailed distributions.

\subsection{Set-Restricted Eigenvalue Condition for heavy-tailed distributions}
        
Most theoretical guarantees for compressed sensing rely on variants
of the Restricted Eigenvalue Condition(REC)~\cite{bickel2009simultaneous,candes2008restricted} and the  
closest to our setting is the Set Restricted 
Eigenvalue Condition~\cite{bora2017compressed}(S-REC). Formally, $A\in\R^{m\times n}$
satisfies S-REC$(S,\gamma,\delta)$ on a set $S\subseteq \R^n$ if for all $x_1, x_2\in S,$
$$\| A x_1 - A x_2 \| \geq \gamma \| x_1 - x_2 \| - \delta.$$

While we can prove many powerful results using the REC condition,
proving that a matrix satisfies REC typically involves sub-Gaussian
entries in $A$. 
If we don't have sub-Gaussianity, proving REC requires a finer analysis.
A recent technique called the \emph{small-ball method} 
~\cite{mendelson2014learning} requires significantly weaker assumptions 
on $A$, and can be used to show 
REC~\cite{mendelson2014learning, tropp2015convex}
for $A$ satisfying~\Cref{assm:assumption}.
While this technique can be used for sparse vectors,
we do not have a general understanding of what structures it can handle,
since existing proofs make heavy use of sparsity.

We now show that a random matrix whose rows
satisfy~\Cref{assm:assumption} will satisfy S-REC 
over the range of a generator 
$G:\R^k\rightarrow \R^n$ with high probability. 
This generalizes Lemma 4.2 in~\cite{bora2017compressed}-- the 
original lemma required i.i.d. sub-Gaussian entries in the matrix 
$A$, whereas the following lemma only needs the rows to have bounded fourth moments.
 
\begin{lemma}
\label{lemma: srec heavy bora}
Let $G:\R^k\rightarrow\R^n$ be a $d-$layered neural network with ReLU activations. Let
$A\in\R^{m\times n}$ be a matrix with i.i.d. rows satisfying \Cref{def:heavy-tailed}.
For any $\gamma < 1,$ if $m=\Omega\left(\frac{1}{1-\gamma^2}kd\log n\right),$ then with probability
$1-e^{-\Omega(m)}$,  for all $z_1, z_2\in\R^k,$ we have
\[
\frac{1}{m}\|AG(z_1) - AG(z_2)\|^2\geq \gamma^2\|G(z_1) - G(z_2)\|^2.
\]
\end{lemma}

This implies that the ERM approach of~\cite{bora2017compressed} still
works when we only have a heavy-tailed measurement matrix $A$. However,
as we show in our experiments, heavy-tailed noise in $y$ and outliers
in $y,A$ will make ERM fail catastrophically. 
In order to solve this problem, we leverage the median-of-means
tournament defined in~\cref{equ:MOM_minmax}, and we will now show it
is robust to heavy-tails and outliers in $y,A$.

\subsection{Main results}

We now present our main result. 
\Cref{thm: mom tournaments} provides recovery guarantees in terms of
the error in reconstruction in the 
presence of heavy-tails and outliers, where $\wh{z}$ is the 
(approximate) minimizer of~\cref{equ:MOM_minmax}. 
First we show that the minimum value of the objective 
in~\cref{equ:MOM_minmax} is indeed small if there are no outliers.

\begin{lemma}\label{lemma: obj min value}
Let $M$ denote the number of batches.
Assume that the measurements $y$ and measurement matrix $A$ are drawn from the uncorrupted distribution satisfying~\Cref{assm:assumption}. Then with probability $ 1- e^{-\Omega(M)},$ the objective in~\Cref{equ:MOM_minmax} satisfies
\begin{align}
    \min_{z\in \R^k}\max_{z'\in \R^k} \underset{1\leq j \leq M}{\mathrm{median }} (  \ell_{B_j}(z) - \ell_{B_j}(z')) \leq 4 \sigma^2.
\end{align}
\end{lemma}

We now introduce \Cref{lemma: srec} and \Cref{lemma: multiplier process}, 
which control two stochastic  processes that appear in 
\cref{equ:MOM_minmax}. We show that minimizing the objective 
in~\cref{equ:MOM_minmax} implies that you are close to the unknown vector $G(z^*)$.
Notice  that since $z^*$ is one feasible solution of the inner 
maximization step of $z'$, we can consider $z' = z^*$. Now consider 
the difference of square losses in \cref{equ:MOM_minmax}, which is 
given by:
\begin{align*}
    \ell_{j}(\wh{z}) - \ell_{j}({z}^*) &= \frac{1}{b}\| A_{B_j}G(\wh{z}) - y_{B_j}\|^2
    - \frac{1}{b}\| A_{B_j}G(z^*) - y_{B_j}\|^2,\\
   &= \frac{1}{b}\|A_{B_j} (G(\wh{z}) - G({z^*}))\|^2 
    - \frac{2}{b}\eta_{B_j}^\top (A_{B_j} (G(\wh{z}) - G({z^*}))), 
\end{align*}
where the last line follows from an elementary arithmetic manipulation.

Assume we have the following bounds on a majority of batches:
 \begin{align}
     \frac{1}{b}\|A_{B_j} (G(\wh{z}) - G({z^*}))\|^2 &\gtrsim \norm{ G(\wh{z}) - G(z^*) }^2,\\
     - \frac{2}{b}\eta_{B_j}^\top (A_{B_j} (G(\wh{z}) - G({z^*}))) &\gtrsim - 
     \norm{G(\wh{z}) - G(z^*)}.
 \end{align}
 Since the objective is the median of the sum of the above terms, a small value of the objective 
 implies that $\norm{ G(\wh{z}) - G(z^*)}$ is small. 
 We formally show these bounds in~\Cref{lemma: srec},~\Cref{lemma: multiplier process}.

\begin{lemma}
\label{lemma: srec}
Let $G: \R^k \to \R^n$ be a generative model from a $d$-layer neural network using ReLU activations. Let $A\in\R^{m\times n}$ be a matrix with i.i.d. uncorrupted rows satisfying~\Cref{def:heavy-tailed}. 
Let the batch size $b=\Theta\left(C^4\right)$, let the number of batches satisfy $M=\Omega(kd\log n)$, and let $\gamma$ be a constant which depends on the moment constant $C$. 
Then with probability at least $1-e^{-\Omega(m)},$ for all $z_1, z_2\in \R^k$ there exists a set $J\subseteq \left[M\right]$ of cardinality at least $0.9M$ such that
\[
\frac{1}{b}\|A_{B_j}(G(z_1) - G(z_2)) \|^2\geq \gamma^2\|G(z_1) - G(z_2)\|^2\;, \forall j\in J.
\]
\end{lemma}

\begin{lemma}\label{lemma: multiplier process}
Consider the setting of~\Cref{lemma: srec} with measurements satisfying $y = AG(z^*) + \eta$, where $y,A,\eta$ satisfy~\Cref{assm:assumption} with noise variance $\sigma^2$. 
For a constant batch size $b$ and number of batches $M=\Omega(kd\log n)$, with probability at least $1-e^{-\Omega(m)},$ for all $z\in \R^k$ there exists a set $J\subseteq \left[M\right]$ of cardinality at least $0.9M$ such that 
\[
\frac{1}{b} |\eta_{B_j}^T A_{B_j} (G(z) - G(z^*))| \leq \sigma \| G(z) - G(z^*) \| \;, \forall j\in J.
\]
\end{lemma}

The above lemmas do not account for the $\epsilon-$corrupted samples in~\Cref{def:corruption}.
However, since the batch size is constant in both the lemmas, there exists a value of $\epsilon$ such that sufficiently many batches have no corruptions.
Hence we can apply~\Cref{lemma: srec},~\Cref{lemma: multiplier process} to these uncorrupted batches. Using these lemmas with a constant batch size $b$, we obtain \Cref{thm: mom tournaments}.
We defer its proof  to \Cref{sec:main_proof}. 

\begin{theorem}\label{thm: mom tournaments}
Let $G: \R^k \to \R^n$ be a generative model from a $d$-layer neural network using ReLU activations.  
There exists a (sufficiently small) constant fraction $\epsilon$ which depends on the moment constant $C$ 
in~\Cref{def:heavy-tailed} such that the following is true. 
We observe $m = O(kd\log n)$ $\epsilon$-corrupted samples from \Cref{def:corruption}, under 
\Cref{assm:assumption}. For any $z^* \in \R^k$, let $\wh{z}$ minimize the 
objective function given by \cref{equ:MOM_minmax} to within additive $\tau$ of the 
optimum. Then there exists a (sufficiently 
large) constant $c$, such that with  probability at least $1-e^{-\Omega(m)}$, the reconstruction $G(\hat{z})$ satisfies
\[
\|G(\wh{z}) - G(z^*)\|^2 \leq c(\sigma^2 + \tau),
\]
where $\sigma^2$ is the variance of noise under~\Cref{assm:assumption}.
\end{theorem}

We briefly discuss the implications of Theorem~\ref{thm: mom tournaments}, with regards to sample complexity and error in reconstruction.
\paragraph{Sample Complexity.}
Our sample complexity matches that of~\cite{bora2017compressed} up to constant factors. This shows that the minimizer of \cref{equ:MOM_minmax} in the presence of heavy-tails and outliers provides the same guarantees as in the case of ERM with sub-Gaussian measurements. 
\paragraph{Statistical accuracy and robustness.}
Let us analyze the error terms in our theorem. The term $\tau$ is a consequence of the minimization algorithm not being perfect, since it only reaches within $\tau$ of the true minimum. Hence it cannot be avoided. 
The term $\sigma^2$ is due to the noise in measurements. 
In the main result of~\cite{bora2017compressed}, the reconstruction $G(\wh{z})$ has error bounded by
$\|G(\wh{z}) - G(z^*)\|^2 \lesssim \norm{\eta}^2/m + \tau.$\footnote{In~\cite{bora2017compressed}, the bound is stated as $\norm{\eta}^2,$ but our $A$ has a different scaling, and hence the correct bound in our setting is $\norm{\eta}^2/m$.}
This gives the following conditions:
\begin{itemize}[leftmargin=*]
    \item  If $\eta$ is sub-Gaussian with variance $\sigma^2$, then 
    $\norm{\eta}^2/m \approx \sigma^2$ with high probability. Hence our bounds match up to constants.
    \item If higher order moments of $\eta$ do not exist, an application of 
    Chebyshev's inequality says that with probability $1-\delta,$~\cite{bora2017compressed} has $\norm{G(z^*) - G(\wh{z})}^2 \approx \sigma^2/(m \delta)$, and this can be extremely large if we want $\delta = e^{-\Omega(m)}.$
\end{itemize}
Hence our method is clearly superior if $\eta$ only has 
bounded variance, and if $\eta$ is sub-Gaussian, then our 
bounds match up to constants. In the presence of 
corruptions,~\cite{bora2017compressed} has no provable guarantee.

\section{Experiments}
\label{sec:experiments}

\begin{figure}[t]
\centering
\begin{subfigure}{.45\textwidth}
  \centering
  \includegraphics[width=\linewidth, height =.65\textwidth ]{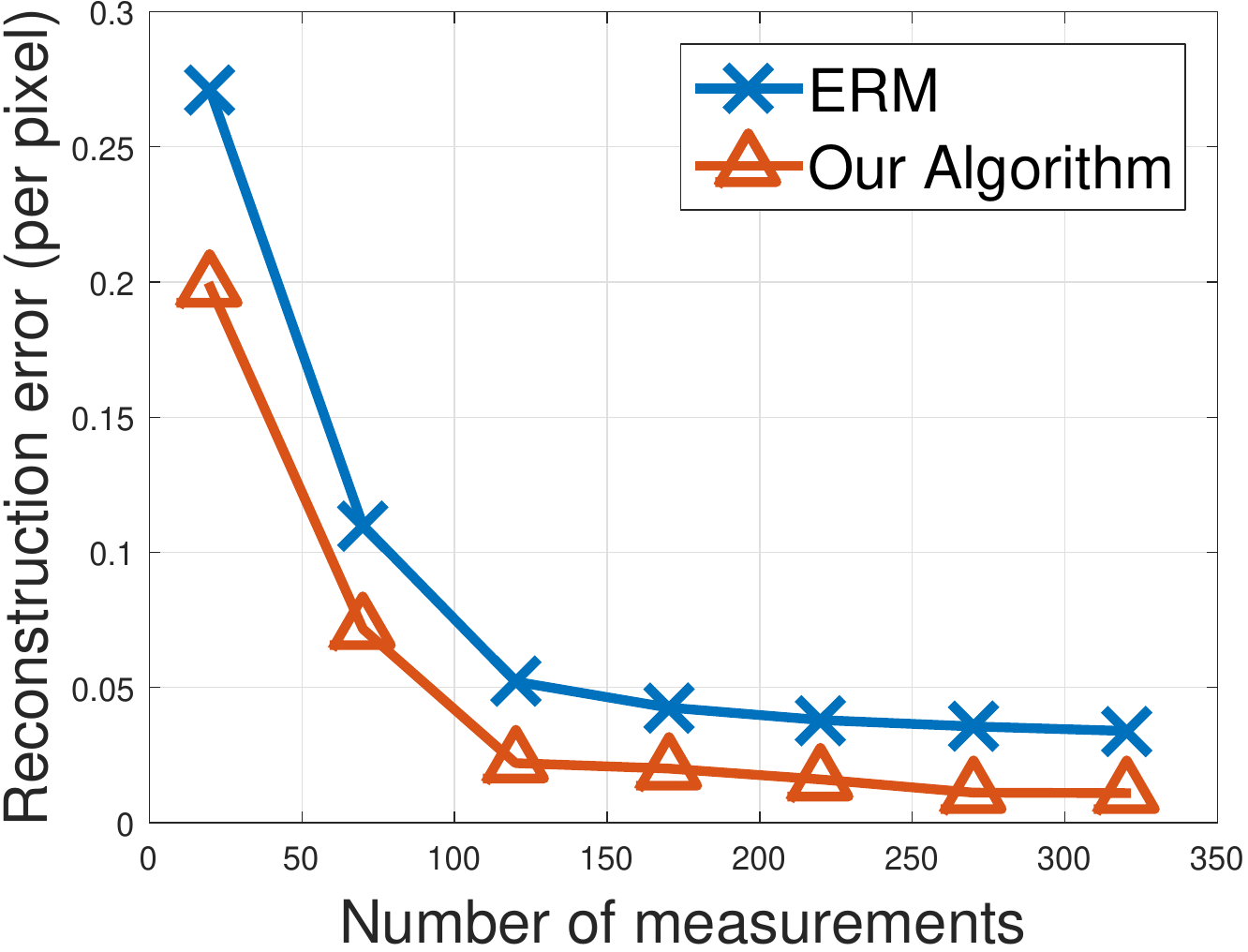}
  \caption{ Results on MNIST. }
  \label{fig:Curve}
\end{subfigure} \hfill
\begin{subfigure}{.45\textwidth}
  \centering
  \includegraphics[width=\linewidth, height =.65\textwidth ]{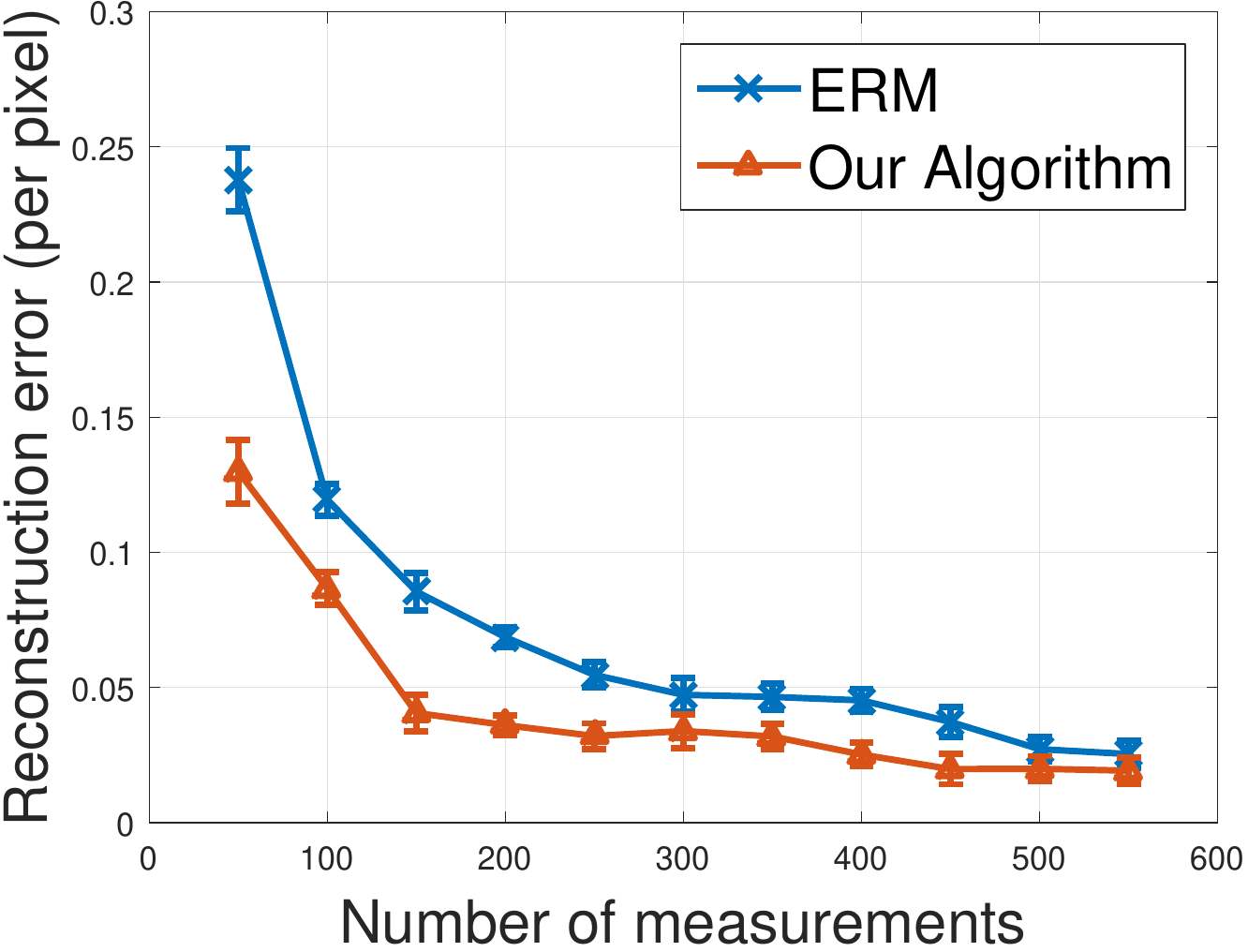}
  \caption{Results on CelebA-HQ}
  \label{fig:Curve_post}
\end{subfigure}
\caption{ \small{We compare \Cref{alg:MOM_GAN} with the baseline ERM \cite{bora2017compressed} under the heavy-tailed setting \emph{without} arbitrary outliers. We fix $k = 100$ for the MNIST dataset
and $k = 512$ for the CelebA-HQ dataset.
We vary the number of measurements, and plot the reconstruction error 
per pixel averaged over multiple trials. 
With increasing number of measurements, we observe the reconstruction 
error decreases.
For heavy-tailed $y$ and $A$ \emph{without} arbitrary outliers, our method obtains significantly smaller reconstruction error in comparison 
to ERM.}}
\label{fig:heavy-l2}
\end{figure}

\begin{figure}[h]
    \centering
    \includegraphics[width=\columnwidth]{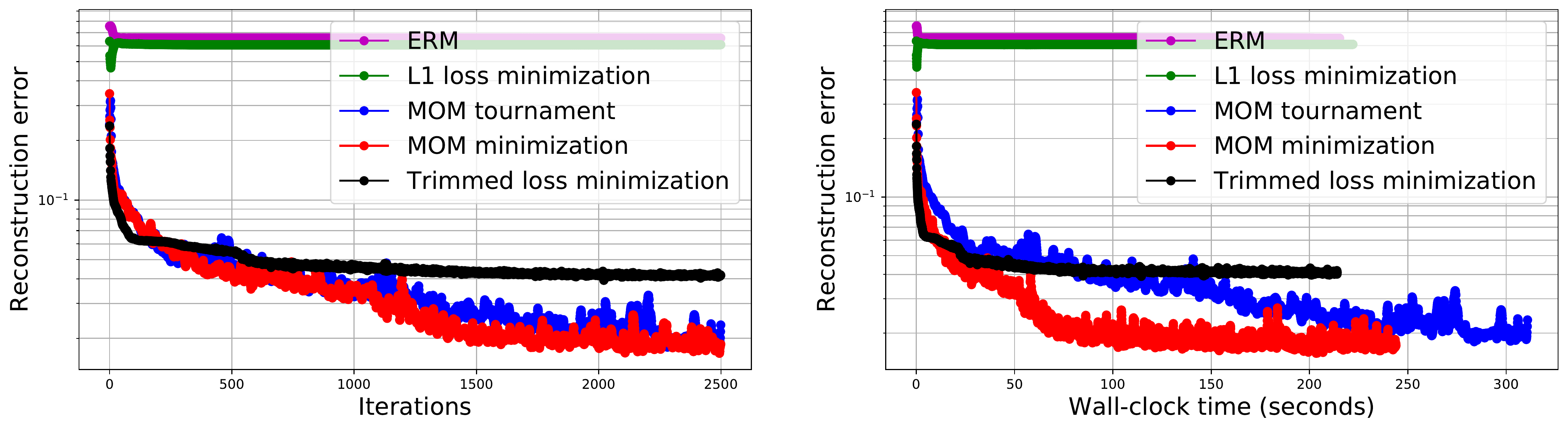}
    \caption{
    \small{ Plot of
the reconstruction error versus the iteration number (left) and plot of the reconstruction error versus
wall-clock time (right).
ERM \cite{bora2017compressed} and $\ell_1$ minimization fail to converge.
Our two proposed methods, MOM tournament(blue) and MOM minimization(red), have 
the smallest reconstruction error. 
We provide a theoretical analysis for the MOM tournament algorithm,
and observe that direct minimization of the MOM objective also works
in practice. The computation time of our algorithms 
is nearly the same as the baselines. 
}}
    \label{fig:Time_Loss}
\end{figure}

In this section, we study the empirical performance of our algorithm 
on generative models 
trained on real image datasets. 
We show that we can reconstruct images under heavy-tailed samples and arbitrary outliers.
For additional experiments and experimental setup details, see \Cref{app:experiments}.

\paragraph{Heavy-Tailed Samples}
In this experiment, we deal with
the \textit{uncorrupted} compressed sensing model $P$, which has heavy-tailed measurement matrix and
stochastic noise: $y = AG(z^*) + \eta$. 
We use a Student’s $t-$distribution (a typical example of heavy-tails) for $A$ 
and $\eta$. 
We compare \Cref{alg:MOM_GAN} with the baseline ERM \cite{bora2017compressed} for heavy-tailed data \emph{without} arbitrary corruptions on MNIST~\cite{lecun1998gradient} and CelebA-HQ~\cite{karras2017progressive,liu2018large}. We trained a DCGAN~\cite{radford2015unsupervised} with $k=100$ and $d=5$ layers to produce $64\times 64$ MNIST images.
For CelebA-HQ, we used a PG-GAN~\cite{karras2017progressive}
with $k=512$ to produce images of size $256\times256\times3 = 196,608$. 

We vary the number of measurements  $m$ and obtain the reconstruction 
error $\|G(\wh{z}) - G(z^*)\|^2/n$ for \Cref{alg:MOM_GAN} and ERM, where
$G(z^*)$ is the ground truth image.
In \Cref{fig:heavy-l2}, \Cref{alg:MOM_GAN} and ERM both have decreasing 
reconstruction error per pixel with increasing number of measurements. 
To conclude, even for heavy-tailed noise \emph{without} arbitrary 
outliers, \Cref{alg:MOM_GAN} obtains significantly smaller reconstruction
error when compared to ERM.

\paragraph{Arbitrary corruptions.}

In this experiment, we use the same heavy-tailed samples as above, and we 
add $\eps = 0.02$-fraction of arbitrary corruption. 
We set the outliers of measurement matrix $A$ as random sign matrix, and the
outliers of $y$ are fixed to be $-1$. We note that we don’t use any targeted attack to simulate the outliers. 
We perform our experiments on the CelebA-HQ dataset using a PG-GAN of latent dimension $k=512$, and fix the number of measurements to $m = 1000$.

We compare our algorithm to a number of natural baselines. Our first 
baseline is ERM~\cite{bora2017compressed} which is not designed to deal 
with outliers. While its fragility is interesting to note, in this sense
it is not unexpected.
For outliers in $y$, classical robust methods replace the 
loss function by an $\ell_1$ loss function or Huber loss function.
This is done in order to avoid the squared loss, which makes recovery 
algorithms very sensitive to outliers. In this case, we have $\wh{z} := \argmin \norm{ y - AG(z)}_1$. 

We also investigate the 
performance of trimmed loss minimization, which is a recent algorithm proposed by~\cite{shensujay2018learning}. This algorithm picks the $t-$fraction of samples with smallest empirical loss for each update step, where $t$ is a hyper-parameter.

We run \Cref{alg:MOM_GAN} and its variant MOM minimization. The MOM 
minimization directly minimizes 
\begin{align}
\label{equ:MOM_direct}
  \widehat{z} = \arg \min_{z\in\R^k} 
{\mathrm{median}}_{1\leq j \leq M} (\ell_j(z)),  
\end{align}
and we use gradient-based methods similar to \Cref{alg:MOM_GAN} to solve it.
Since \Cref{alg:MOM_GAN} optimizes $z$ and $z'$ in one iteration, the actual computation time of MOM tournament is twice that of MOM minimization.
As shown in~\Cref{fig:Time_Loss},~\Cref{fig:celebA},
ERM \cite{bora2017compressed} and $\ell_1$ loss minimization fail to converge to the ground truth and in particular, 
they may recover a completely different person. 
Trimmed loss minimization~\cite{shensujay2018learning} 
only succeeds on occasion,
and when it fails, it obtains a visibly different person. 
The convergence of the MOM minimization per iteration is very
similar to the MOM tournament, and they both achieve much smaller
reconstruction error compared to trimmed loss minimization. The 
right panel of \Cref{fig:Time_Loss} plots the reconstruction error 
versus the actual computation time, showing our
algorithms match baselines. We plot the MSE vs. number of measurements in Figure~\ref{fig:additional-expts2}, where the fraction of corruptions is set to $\eps=0.02$.

\paragraph{Miscellaneous Experiments}

\emph{Is ERM ever better than MOM?}
So far we have analyzed cases where MOM performs better than ERM.
Since ERM is known to be optimal in linear regression when dealing with uncorrupted sub-Gaussian data, we expect it to be superior to MOM when our measurements are all sub-Gaussian.
We evaluate this in Fig.~\ref{fig:additional-expts1} and observe that ERM obtains smaller MSE in this setting. Notice that as we reduce the number of batches in MOM, it approaches ERM.

\emph{How sensitive is MOM to the number of batches?}
In Figure~\ref{fig:additional-expts3} we study the MSE of MOM tournaments and MOM minimization as we vary the number of batches.

In order to select the optimal number of batches ($M$), we keep a set of validation measurements that we do not use in the optimization routines for estimating $x$. We can run MOM for different value of $M$ to get multiple reconstructions, and then evaluate each reconstruction using the validation measurements to pick the best reconstruction.
Note that one should use the median-of-means loss while evaluating the validation error as well.

\begin{figure}[t]
    \includegraphics[width=\textwidth]{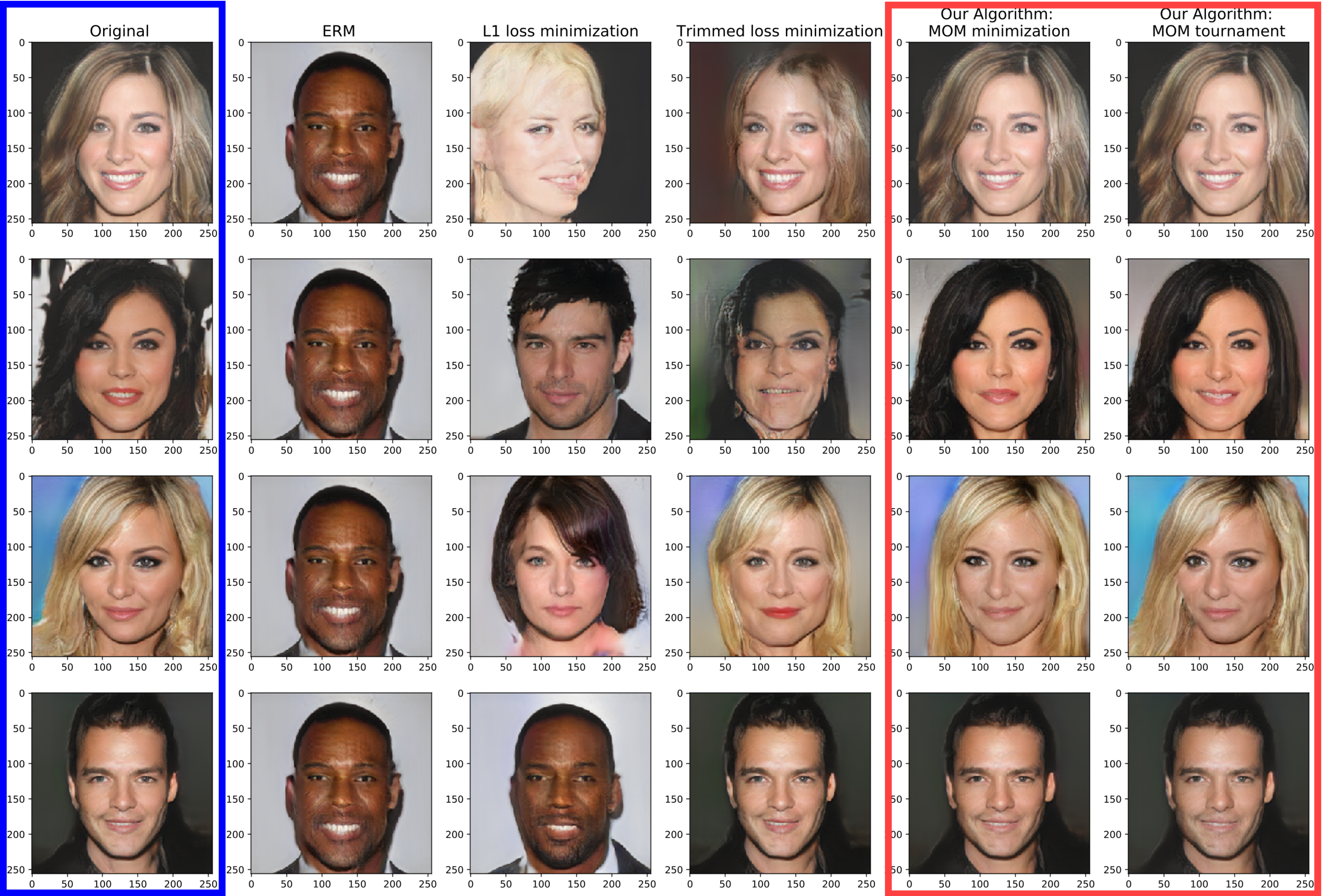}
    \caption{
    \small
    Reconstruction results on CelebA-HQ for $m = 1000$ measurements with 20 corrupted measurements.
    For each row, the first column is ground truth from a 
    generative model. 
    Subsequent columns show reconstructions by 
    ERM~\cite{bora2017compressed}, $\ell_1-$minimization, trimmed 
    loss minimization~\cite{shensujay2018learning}.
    In particular, vanilla ERM,  $\ell_1-$minimization obtain 
    completely different faces. Since we use the same outlier for 
    different rows, vanilla ERM produces the same reconstruction
    irrespective of the ground truth.
    Trimmed loss minimization only succeeds on occasion (the last
    row),
    and when it fails, it obtains a similar but still different 
    face. 
    The last two columns show reconstructions by our proposed  
    algorithms. The second to last one is directly minimizing the 
    MOM objective \cref{equ:MOM_direct}, and the last column 
    minimizes the MOM tournament objective \cref{equ:MOM_minmax}. We
    provide a theoretical analysis for the MOM tournaments 
    algorithm, and observe that direct minimization of the MOM 
    objective also works in practice. We observe the last two 
    columns have much better reconstruction performance -- we get a 
    high quality reconstruction under heavy-tailed 
    measurements and arbitrary outliers.}
    \label{fig:celebA}
\end{figure}

\begin{figure}[b]
\centering
\begin{subfigure}{.31\textwidth}
  \centering
  \includegraphics[width=\linewidth, height =.65\textwidth ]{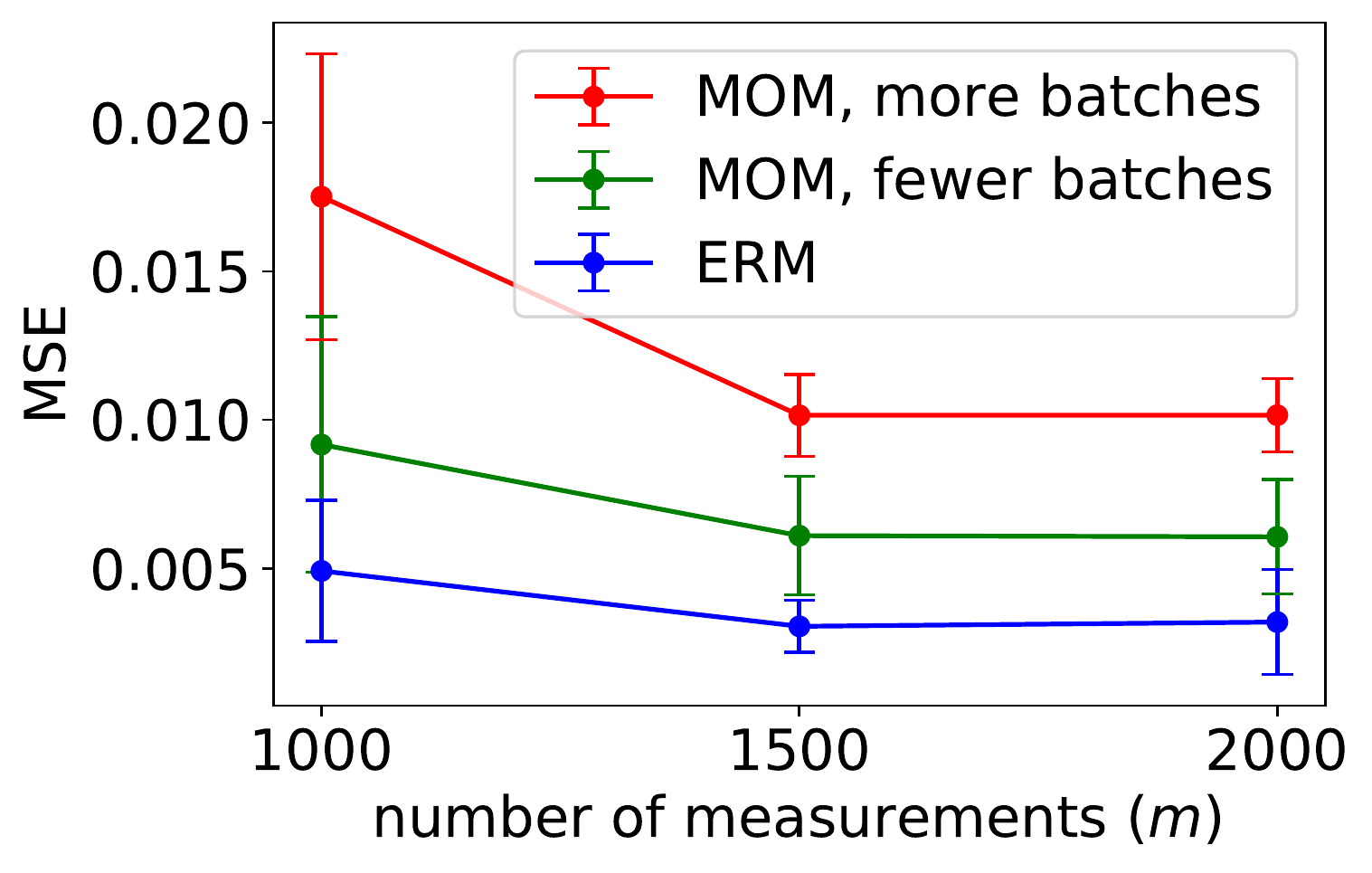}
  \caption{\small ERM vs MOM}
  \label{fig:additional-expts1}
\end{subfigure}
\hfill
\begin{subfigure}{.31\textwidth}
  \centering
  \includegraphics[width=\linewidth, height =.65\textwidth ]{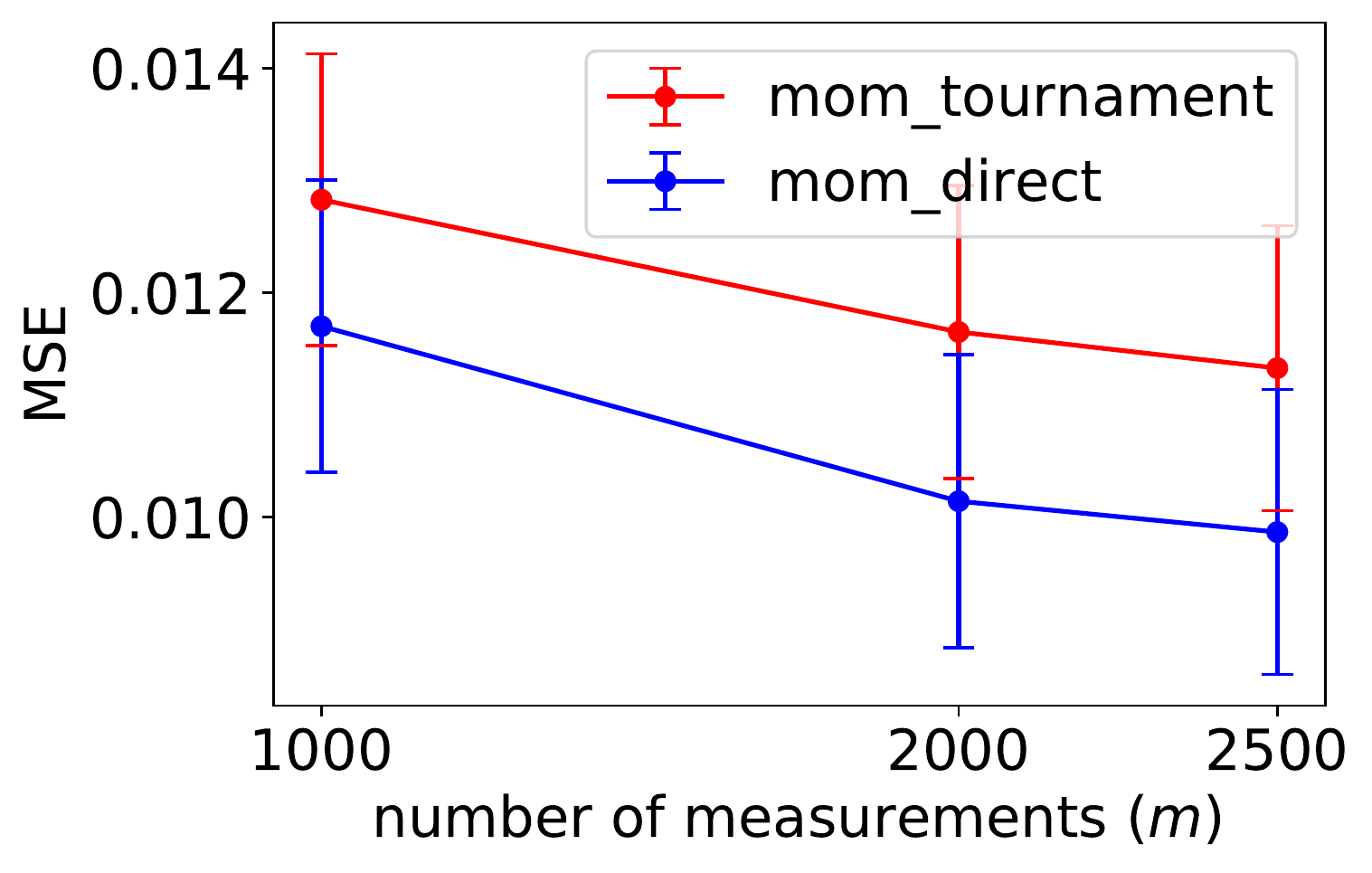}
  \caption{\small MOM on arbitrary corruptions}
  \label{fig:additional-expts2}
\end{subfigure} 
 \hfill
\begin{subfigure}{.31\textwidth}
  \centering
  \includegraphics[width=\linewidth, height =.65\textwidth ]{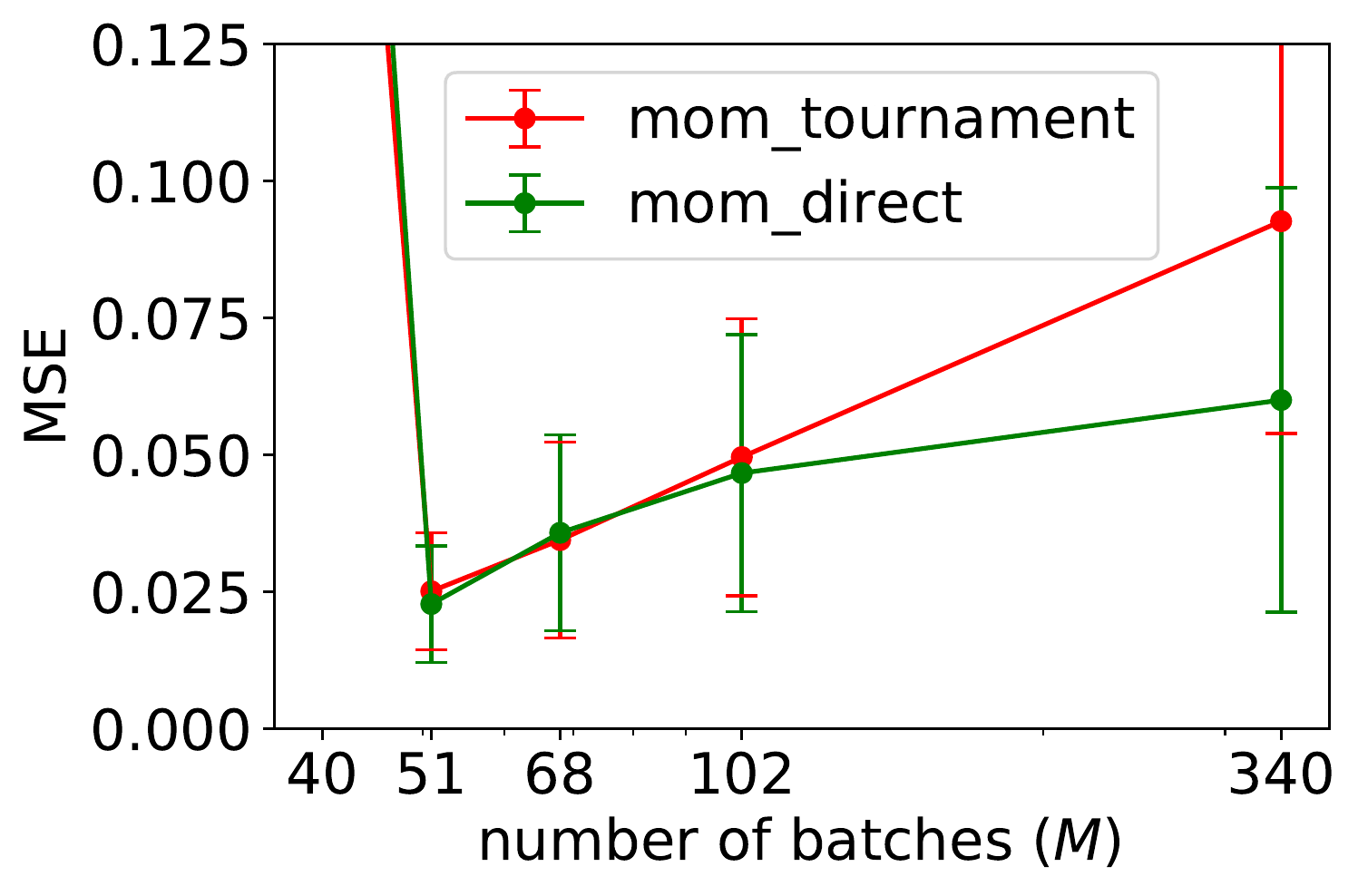}
  \caption{\small MSE vs. Number of batches}
  \label{fig:additional-expts3}
\end{subfigure} 
\caption{\small (a) We compare ERM and MOM by plotting MSE vs number of measurements when the measurements are \emph{sub-Gaussian without corruptions}.  (b) Aggregate statistics for MOM in the presence of corruptions. (c) MSE vs number of batches for MOM on $1000$ heavy-tailed measurements and 20 corruptions. All error bars indicate 95\% confidence intervals. Plots use a PGGAN on CelebA-HQ.} 
\label{fig:additional-expts}
\end{figure}

\section{Conclusion}

The phenomenon observed in~\Cref{fig:celebA} highlights the 
importance of our method. Our work raises several
questions about why the objective we
consider can be minimized, and suggests we need a new
paradigm for analysis that accounts for similar instances
that enjoy empirical success, even though they can be
provably hard in the worst case.

\section{Acknowledgments}
Ajil Jalal and Alex Dimakis have been supported by NSF Grants CCF 1763702,1934932, AF 1901292,
2008710, 2019844 research gifts by NVIDIA, Western Digital, WNCG IAP, computing resources from TACC and the Archie Straiton Fellowship.
Constantine Caramanis and Liu Liu have been supported by NSF Award 1704778 and a grant from the Army Futures Command.

\bibliography{main}
\bibliographystyle{plain}

\newpage

\appendix

\section{Proof of \Cref{lemma: srec heavy bora}}

\begin{lemma*}[\Cref{lemma: srec heavy bora}]
Let $G:\R^k\rightarrow\R^n$ be a $d-$layered neural network with ReLU activations. Let $A\in\R^{m\times n}$ be a matrix with i.i.d rows satisfying \Cref{assm:assumption}. If $m=\Omega\left(\frac{1}{1-\gamma^2}kd\log n\right),$ then with probability $1-e^{-\Omega(m)}$, $A$ satisfies
\[
\frac{1}{m}\|AG(z_1) - AG(z_2)\|^2\geq \gamma^2\|G(z_1) - G(z_2)\|^2 
\]
for all $z_1, z_2\in\R^k.$
\end{lemma*}

\begin{proof}
The proof is based on \Cref{prop: relu no of subspaces} and 
\Cref{prop: srec single subspace},
which will be introduced as follows.
\Cref{prop: relu no of subspaces} shows that the set $S_G=\{G(z_1) - G(z_2): z_1, z_2 \in\R^k\}$ lies in the range of $e^{O(kd\log n)}$ different $2k-$dimensional subspaces.

\Cref{prop: srec single subspace} guarantees the result for a single subspace with probability $1-e^{-m}$. Since $m = \Omega(kd \log n)$, the proof follows from a union bound over the $e^{O(kd\log n)}$ subspaces in \Cref{prop: relu no of subspaces}.
\end{proof}

\begin{proposition}
\label{prop: relu no of subspaces}
If $G:\R^k\rightarrow\R^n$ is a $d-$layered neural network with ReLU activations, then the set $S_G=\{G(z_1) - G(z_2): z_1, z_2 \in\R^k\}$ lies in the union of $O\left(n^{2kd}\right)$ different $2k-$dimensional subspaces.
\end{proposition}
\begin{proof}[Proof of Proposition~\eqref{prop: relu no of subspaces}]
From Lemma 8.3 in~\cite{bora2017compressed}, the set 
$\{G(z): z\in\R^k\}$
lies in the union of $O(n^{kd})$ different $k-$dimensional subspaces. 

This implies that the set 
\[
\{G(z_1) - G(z_2): z_1, z_2\in\R^k\}
\]
lies in the union of $M = O(n^{2kd})$ different $2k-$dimensional subspaces. 

\end{proof}

 \begin{proposition}
 \label{prop: srec single subspace}
 Consider a single $2k-$dimensional subspace given by $S_1=\{Wz: W\in\R^{n\times 2k}, W^T W=I_{2k}, z\in\R^{2k}\}$. 
 Let $A\in\R^{m\times n}$ be a matrix with i.i.d rows drawn from a 
 distribution satisfying Assumption \eqref{assm:assumption}. 
 If $m=O(\frac{C^2 k}{\frac{3}{4}-\gamma^2})$, with probability 
 $1-e^{-\Omega(m)},$ $A$ satisfies
 \[
 \frac{1}{m}\|Av\|^2 \geq \gamma^2 \|v\|^2,\,\forall v\in S_1.
 \]
 \end{proposition}

 \begin{proof}
 The proof follows Theorem 14.12 in~\cite{wainwright2019high}, with
 non-trivial modifications for our setting.

 We want to show that for all vectors $v\in S_1$,
 \[
 \frac{1}{m}||Av||^2 \geq \gamma^2 ||v||^2.
 \]

 For $u,\tau\in\R$, define the truncated quadratic function
 \begin{equation}\label{eqn: definition of phi}
 \phi_\tau(u) = \begin{cases}
 & u^2 \quad \text{if } |u|\leq \tau,\\
 & \tau^2 \quad \text{otherwise.}
 \end{cases}
 \end{equation}

 By construction,
 $
 \phi_\tau(\inn{a_i}{v}) \leq \inn{a_i}{v}^2.
 $

 This implies that
 \begin{align}
     &\frac{1}{m}||Av||^2 = \frac{1}{m}\sum_{i=1}^m\inn{a_i}{v}^2 = \frac{\|v\|^2}{m}\sum_{i=1}^m\inn{a_i}{\tfrac{v}{\| v \|}}^2 \\
     &\geq \frac{\|v\|^2}{m}\sum_{i=1}^m \phi_\tau(\inn{a_i}{\tfrac{v}{\| v \|}})\\
     &\geq \| v \|^2\E\left[\frac{\sum_{i=1}^m \phi_\tau(\inn{a_i}{\tfrac{v}{\| v \|}})}{m}\right] -  \|v\|^2\left|\frac{\sum_{i=1}^m \phi_\tau(\inn{a_i}{\tfrac{v}{\| v \|}})}{m} - \E\left[\frac{\sum_{i=1}^m \phi_\tau(\inn{a_i}{\tfrac{v}{\| v \|}})}{m}\right] \right| \\
     &= \| v \|^2\E\left[\phi_\tau(\inn{a}{\tfrac{v}{\| v \|}})\right] -  \|v\|^2\left|\frac{\sum_{i=1}^m \phi_\tau(\inn{a_i}{\tfrac{v}{\| v \|}})}{m} - \E\left[\phi_\tau(\inn{a}{\tfrac{v}{\| v \|}})\right] \right| \\
     &\geq \| v \|^2\E\left[\phi_\tau(\inn{a}{\tfrac{v}{\| v \|}})\right] - \|v\|^2\sup_{v\in S_1}\left|\frac{1}{m}\sum_{i=1}^m \phi_\tau(\inn{a_i}{\tfrac{v}{\| v \|}}) - \E\left[\phi_\tau(\inn{a}{\tfrac{v}{\| v \|}})\right] \right|
 \end{align}

 In \Cref{claim: relu paley} we will show that for $\tau^2 = \frac{C^4}{\frac{3}{4} - \gamma^2}$, we have
     \[\E\left[\phi_\tau(\inn{a}{\tfrac{v}{\| v \|}})\right] \geq(\gamma^2 + \frac{1}{4}).\]

 In \Cref{claim: relu local rademacher} we will show that with overwhelming probability in $m$,
      \[\sup_{v: \|v\|\leq 1}\left|\frac{1}{m}\sum_{i=1}^m \phi_\tau(\inn{a_i}{\tfrac{v}{\| v \|}}) - \E\left[\phi_\tau(\inn{a}{\tfrac{v}{\| v \|}})\right] \right| \leq \frac{1}{4}.\]
   
 These two results together imply that
 \[
 \frac{1}{m}\|Av\|^2 \geq \gamma^2 \|v\|^2.
 \] with overwhelming probability in $m$.
 \end{proof}

 \begin{claim}
 \label{claim: relu paley}

 Assume that the random vector $a$ satisfies Assumption~\eqref{assm:assumption} with constant $C$.
 Let $\phi_\tau$ be the thresholded quadratic function defined in Eqn~\eqref{eqn: definition of phi}.
 For all $v\in\R^n, \|v\|\leq 1,$ we have
 \[\E\left[\phi_\tau(\inn{a}{v})\right] \geq\left(1-\frac{C^4}{\tau^2}\right)\|v\|^2.
 \]
 \end{claim}

 \begin{proof}
 \begin{align}\label{proof: Ephi part 1}
     \|v\|^2 - \E\left[\phi_\tau(\inn{a}{v})\right] =& \E\left[\inn{a}{v}^2\right] - \E\left[\phi_\tau(\inn{a}{v})\right]\\
     =&  \E\left[(\inn{a}{v}^2 - \tau^2)1_{\{|\inn{a}{v}|\geq \tau\}}\right]\\
     \leq& \E\left[\inn{a}{v}^2 1_{\{|\inn{a}{v}|\geq \tau\}}\right]
 \end{align}

 By the Cauchy-Schwartz inequality,
 \begin{align}\label{eqn: holder}
 &\E\left[\inn{a}{v}^2 1_{\{|\inn{a}{v}|\geq \tau\}}\right] \leq\left(\E\left[\inn{a}{v}^4\right]\right)^{\frac{1}{2}}\left(\Pr\left[|\inn{a}{v}|\geq \tau\right]\right)^{\frac{1}{2}}
 \end{align}

 From Assumption \eqref{assm:assumption}, we have
 \[
 \left(\E\left[\inn{a}{v}^4\right]\right)^{\frac{1}{2}}\leq C^2\E\left[\inn{a}{v}^2\right].
 \]

 From Chebyshev's inequality and Assumption \eqref{assm:assumption}, we have
 \begin{align}
 \left(\Pr\left[|\inn{a}{v}|\geq \tau\right]\right)^{\frac{1}{2}}
 &\leq \left(\frac{\E\left[|\inn{a}{v}|^4\right]}{\tau^4}\right)^{\frac{1}{2}} \leq \left(\frac{C^4\E\left[|\inn{a}{v}|^2\right]^2}{\tau^4}\right)^{\frac{1}{2}} = \frac{C^2 \E\left[|\inn{a}{v}|^2\right]}{\tau^2}.
 \end{align}

 Substituting the above two inequalities into \cref{eqn: holder}, we get
 \begin{align}
   \E\left[\inn{a}{v}^2 1_{\{|\inn{a}{v}|\geq \tau\}}\right] & \leq \frac{C^4 \E\left[\inn{a}{v}^2\right]^2}{\tau^{2}} \\
   &= \frac{C^4\|v\|^4}{\tau^{2}}  \leq \frac{C^{4}\|v\|^2 }{\tau^2}.
 \end{align}

 Substituting into Eqn~\eqref{proof: Ephi part 1},
 \begin{align}
     \|v\|^2 - \E\left[\phi_\tau(\inn{a}{v})\right] \leq \frac{C^4\|v\|^2}{\tau^2},
 \end{align}
 which completes the proof.
 \end{proof}

 \begin{claim}
 \label{claim: relu local rademacher}
 For an orthonormal matrix $U\in\R^{n\times 2k}$, let $S:=\{v: v=Uz, \|v\|=1\}$. Let $\phi_\tau$ be the function defined in  \Cref{prop: srec single subspace}. For $m=\Omega\left(\tau^2 k\right),$ we have
      \[\sup_{v\in S}\left|\frac{1}{m}\sum_{i=1}^m \phi_\tau(\inn{a_i}{v}) - \E\left[\phi_\tau(\inn{a}{v})\right] \right| \leq \frac{1}{4}.\]
     with probability $1-e^{-\Omega (m)}.$
 \end{claim}

 \begin{proof}
 Define 
 \[
 Z_m = \sup_{v\in S}\left|\frac{1}{m}\sum_{i=1}^m \phi_\tau(\inn{a_i}{v}) - \E\left[\phi_\tau(\inn{a}{v})\right] \right|.
 \]

 We will first show that 
 \[\E_A\left[Z_m\right]\leq \frac{1}{8}\]
 for large enough $m$. Then we use Talagrand's inequality~\cite{talagrand1991new} to show that \[\Pr\left[Z_m\geq \E\left[Z_m\right] + \frac{1}{8}\right]\leq e^{-\Omega(m)},\]
 using which we can conclude that $Z_m\leq \frac{1}{4}$ with probability $ 1-e^{-\Omega(m)}.$

 By the symmetrization inequality, we have
 \[
 \E_{A}\left[Z_m\right] \leq 2\E_{\epsilon,A}\left[\sup_{v\in S}\left|\frac{1}{m}\sum_{i=1}^m \epsilon_i \phi_\tau(\inn{a_i}{v})\right|\right]
 \]
 where $\{\epsilon_i\}_{i=1}^m$ are i.i.d Bernoulli $\pm 1$ random variables.

 Since $\phi_\tau$ is a Lipschitz function with Lipschitz constant $2\tau$, we can apply the Ledoux-Talagrand contraction inequality~\cite{ledoux2013probability} (refer to \Cref{sec:background} for the sake of completeness) to get
 \begin{align}
 &2\E_{\epsilon,A}\left[\sup_{v\in S}\left|\frac{1}{m}\sum_{i=1}^m \epsilon_i \phi_\tau(\inn{a_i}{v})\right|\right] \nonumber\\
 \leq & 8\tau \E_{\epsilon,A}\left[\sup_{v\in S}\left|\frac{1}{m}\sum_{i=1}^m \epsilon_i \inn{a_i}{v}\right|\right]\\
 =&8\tau\E_{\epsilon,A}\left[\sup_{v\in S }\left|\frac{1}{m}\epsilon^T Av\right|\right].
 \end{align}

 Since $S:=\{v: v=Uz, \|v\|=1\}$, we have 
 \begin{align}
     &8\tau\E_{\epsilon,A}\left[\sup_{v\in S }\left|\frac{1}{m}\epsilon^T Av\right|\right]\\
     =& 8\tau\E_{\epsilon,A}\left[\sup_{z:\|z\|=1}\left|\frac{8\tau}{m}\epsilon^T AUz\right|\right]\\
     \leq& \frac{8\tau}{m} \E_{\epsilon,A}\left[\|\epsilon^T AU\|_2\right]\\
     \leq & \frac{8\tau}{m} \sqrt{\E_{\epsilon,A}\left[\|\epsilon^T AU\|_2^2\right]}
 \end{align}
 The third line follows from the Cauchy-Schwartz inequality, and the fourth line follows from Jensen's inequality.

 Notice that 
 \[
 \E_{\epsilon}\left[\|\epsilon^T AU\|_2^2\right] = \text{trace}(AU U^T A^T) = \text{trace}(U^T A^T AU)  
 \]

 Since $U^TU = I_{2k}$, we have 
 \begin{align}
 &\E_{\epsilon,A}\left[\|\epsilon^T AU\|_2^2\right] = \E_A\left[\text{trace}(U^T A^T AU)  \right] \\
 &= \sum_{i=1}^m \E_{a_i}\text{trace}(U^T a_i a_i^T U) \\
 &= \sum_{i=1}^m \text{trace}(U^T I_n U) = m\text{ trace}(I_{2k})= 2km.
 \end{align}

 Putting this together, and choosing $m=\Omega(\tau^2 k),$ we have
 \[
 \E_A\left[Z_m\right] \leq 8\tau \sqrt{\frac{2k}{m}}\leq \frac{1}{8}.
 \]

 We now need to show that
 \[\Pr\left[Z_m\geq \E\left[Z_m\right] + \frac{1}{8}\right]\leq e^{-\Omega(m)}.\]

 By construction, $\phi_\tau(\inn{a_i}{v})\leq \tau^2$ for all $v\in S$.

 In order to apply Talagrand's inequality, we need to bound $$\sigma^2=\sup_{v\in S}\E\left[\left(\phi_\tau(\inn{a}{v})-\E\left[\phi_\tau(\inn{a}{v})\right]\right)^2\right].$$ We can bound this by
 \begin{align}
     \text{var}(\phi_\tau(\inn{a}{v}) &\leq \E \left[\phi^2_\tau(\inn{a}{v})\right] \\
     &\leq \tau^2\E \left[\phi_\tau(\inn{a}{v})\right] \leq \tau^2
 \end{align}

 Applying Talagrand's inequality, we have
 \[\Pr\left[Z_m\geq \E\left[Z_m\right] + t\right]\leq C_1 \exp\left({-\frac{C_2 mt^2}{\tau^2 + \tau^2 t }}\right).\]

 Setting $t=\frac{1}{8},m=\Omega(\tau^2 k)$ we get 
 \[\Pr[Z_m \geq \frac{1}{4}] \leq \Pr\left[Z_m\geq \E\left[Z_m\right] + \frac{1}{8}\right]\leq C_1 e^{-\frac{C_2 m}{\tau^2}}=e^{-\Omega(m)}.\]

 This concludes the proof.
 \end{proof}

\section{Proof of~\Cref{lemma: obj min value}}
\begin{lemma}
Let $M$ denote the number of batches. Then with probability $ 1- e^{-\Omega(M)},$ the objective in~\Cref{equ:MOM_minmax} satisfies
\begin{align}
    \min_{z\in \R^k}\max_{z'\in R^k} \underset{1\leq j \leq M}{\mathrm{median }} \ell_{B_j}(z) - \ell_{B_j}(z') \leq 4 \sigma^2.
\end{align}
\end{lemma}
\begin{proof}
By setting $z \leftarrow z^*$, for all $z'\in \R^k,$ for any $j\in [M]$, we have
\begin{align}
\ell_{B_j}(z^*) - \ell_{B_j }(z') & \leq \ell_{B_j}(z^*) = \frac{1}{b}\norm{ \eta_{B_j}}^2. 
\end{align}

Since the noise is i.i.d. and has variance $\sigma^2$, we have $\E\left[ \ell_{B_j}(z^*) \right] = \E \frac{1}{b}\norm{\eta_{B_j}}^2 =\sigma^2.$

For batch $j\in [M],$ define the indicator random variable
\[
Y_j = \bm{1}\left\{ \ell_{B_j}(z^*) \geq 4 \sigma^2 \right\}.
\]

By Markov's inequality, since $\E [\ell_{B_j}(z^*)] = \sigma^2,$ we have
\begin{align}
    \Pr \left[ Y_j = 1 \right] \leq \frac{1}{4 } \Rightarrow \E \left[\sum_{j=1}^M Y_j \right] \leq \frac{M}{4}.
\end{align}

By the Chernoff bound, 
\begin{align}
    \Pr \left[ \sum_{j=1}^M Y_j \geq \frac{M}{2}\right] & \leq \Pr \left[ \sum_{j=1}^M Y_j \geq 2 \E[ \sum_{j=1}^M Y_j]\right] \leq e^{-\Omega(M)}.
\end{align}

The above inequality implies that with probability $ 1 - e^{-\Omega(M)},$ for all $z'\in \R^k,$ at least $\frac{M}{2}$ batches satisfy 
\[
\ell_{B_j} (z^*) - \ell_{B_j} (z') \leq 4 \sigma^2.
\]

This gives
\begin{align}
    \min_{z\in \R^k}\max_{z'\in R^k} \underset{1\leq j \leq M}{\mathrm{median }}( \ell_{B_j}(z) - \ell_{B_j}(z')) \leq 4 \sigma^2.
\end{align}

\end{proof}

\section{Proof of \Cref{lemma: srec}}
\begin{lemma*}[\Cref{lemma: srec}]
Let $G: \R^k \to \R^n$ be a generative model from a $d$-layer neural network using ReLU activations. Let $A\in\R^{m\times n}$ be a matrix with i.i.d rows satisfying~\Cref{assm:assumption}. 
Let the batch size $b=\Theta\left(C^4\right)$, let the number of batches satisfy $M=\Omega(kd\log n)$, and let $\gamma$ be a constant which depends on the moment constant $C$. 
Then with probability at least $1-e^{-\Omega(m)},$ for all $z_1, z_2\in \R^k$ there exists a set $J\subseteq \left[M\right]$ of cardinality at least $0.9M$ such that
\[
\frac{1}{b}\|A_{B_j}(G(z_1) - G(z_2)) \|^2\geq \gamma^2\|G(z_1) - G(z_2)\|^2\;, \forall j\in J.
\]

\end{lemma*}

\begin{proof}
  \Cref{prop: relu no of subspaces} shows that the set $S_G=\{G(z_1) - G(z_2): z_1, z_2 \in\R^k\}$ lies in the range of $e^{O(kd\log n)}$ different $2k-$dimensional subspaces.

  \Cref{prop: mom srec single subspace} guarantees the result for a single subspace with probability $1-e^{-\Omega(M)}$. Since $M = \Omega(kd \log n)$ and the batch size is constant which depends on the moment constant $C$, the lemma follows from a union bound over the $e^{O(kd\log n)}$ subspaces in \Cref{prop: relu no of subspaces}.
\end{proof}

 \begin{proposition}
 \label{prop: mom srec single subspace}
 Consider a single $2k-$dimensional subspace given by $S=\{Wz: W\in\R^{n\times 2k}, W^T W=I_{2k}, z\in\R^{2k}\}$. 
 Let $A\in\R^{m\times n}$ be a matrix with i.i.d rows drawn from a distribution satisfying Assumption \eqref{assm:assumption} with constant $C$.
 If the batch size $b = O(C^4)$ and the number of batches satisfies $ M = \Omega\left( k \log \frac{1}{\eps} \right) $, with probability $1-e^{-\Omega(M)},$ for all $x \in S$, there exist a subset of batches $J_x \subseteq [M]$ with $|J_x|\geq 0.90M$ such that 
 \[
 \frac{1}{b}\|A_{B_j} x\|^2 \geq \gamma^2 \|x\|^2 \,\forall j\in J_x,
 \]
 where $\gamma = \Theta\left(\frac{1}{C^2}\right)$ is a constant that depends on the moment constant $C$.
 \end{proposition}

\begin{proof}
  Since the bound we want to prove is homogeneous, it suffices to show it for all vectors in $S$ that have unit norm.
  Let $W \in \R^{n \times 2k}$ be the orthonormal matrix spanning $S$, and $S_1$ denote the set of unit norm vectors in its span. 
  That is, 
  \[
  S_1 = \{ Wz: z\in \R^{2k}, \| z \| = 1, W\in\R^{n\times 2k}, W^T W = I_{2k}\}.
  \]

  For a fixed $x\in S_1$ and $0< t < 1$, we have
  \begin{align}
    \E \left[ \inn{a}{x}^2 \right] =& \E\left[\inn{a}{x}^2 \bm{1}\{\inn{a}{x} \leq t^2\|x\|^2  \}\right] \E\left[\inn{a}{x}^2 \bm{1}\{\inn{a}{x} > t^2\|x\|^2 \}\right] \\
    & \leq t^2 \|x\|^2  + \E\left[\inn{a}{x}^4\right]^{\frac{1}{2}} \left( \Pr \left[ \inn{a}{x}^2 \geq t^2 \| x \|^2   \right]  \right)^{\frac{1}{2}}\\
    & \leq t^2 \|x\|^2  + C^2 \| x \|^2 \left( \Pr \left[ \inn{a}{x}^2 \geq t^2 \| x \|^2   \right]  \right)^{\frac{1}{2}}\\
     \Rightarrow  \Pr \left[ \inn{a}{x}^2 \geq t^2 \| x \|^2 \right]  &\geq \frac{\left( 1 - t^2 \right)^2 \| x \|^4 }{C^4 \| x \|^4 } = \frac{\left( 1 - t^2 \right)^2 }{C^4} = C_1.
  \end{align}
  This is essentially a modified version of the Paley-Zigmund inequality~\cite{paley1932note}.

  Consider a batch $B_j$, which has $b$ samples.
  By the concentration of Bernoulli random variables, with probability $ 1 - 2e^{- \Omega \left( {C_1b} \right)},$ we have 
  $$ \sum_{i\in B_j} \bm{1}\left\{ \inn{a_i}{x}^2 \geq t^2 \| x \|^2  \right\} \geq \frac{bC_1}{2}$$

  This implies that if we set $b$ such that $1 - 2e^{-\Omega \left( C_1 b \right)} = 0.975$, then with probability 0.975, $B_j$ has $\frac{b C_1}{2}$ samples $\inn{a_i}{x}$ whose magnitude is at least $ t \| x \|$.
  This implies that the average square magnitude over the batch satisfies  
  \begin{align}
    \frac{1}{b}\|A_{B_j}x\|^2 =  \frac{1}{b} \sum_{i\in B_j} \inn{a_i}{x}^2 \geq t^2 \| x \|^2 \frac{ b C_1 }{2b} = \frac{C_1 t^2 \| x \|^2 }{2},
    \label{eqn: fixed x srec paley}
  \end{align}
  with probability $0.975$. 
  
  Consider the indicator random variable associated with the complement of the above event.
  That is,
  $$ Y_j(x) = \bm\left\{\frac{1}{b}\|A_{B_j}x\|^2 \leq \frac{C_1 t^2}{2} \| x \|^2  . \right\}$$

  From \eqref{eqn: fixed x srec paley} we have that $\E\left[ Y_j(x) \right] \leq 0.025 $.

  Consider the sum of indicator random variables over $M$ batches.
  By standard concentrations of Bernoulli random variables, we have with probabibility $1-e^{-\Omega\left( M \right)}$, 
  $$ \sum_{j = 1}^M Y_j(x) \leq 2\E\left[\sum_{j=1}^M Y_j(x) \right] \leq 0.05.$$ 

  This implies that there exist a subset of batches $J \subseteq \left[ M \right]$ with $|J| \geq 0.95M$ such that
  $$  
  \frac{1}{b} \| A_{B_j} x \|^2 \geq \frac{C_1 t^2 \| x \|^2 }{2} \;\forall \;j\in J,
  $$
  with probability $1-e^{-\Omega(M)}$. 
  This shows that we have the statement of the proposition for a fixed vector in $S_1$.

  We now show that this holds true for an $\eps-$cover of $S_1$.
  Let $S_\eps$ denote a minimial $\eps-$covering of $S_1$.
  That is, $S_\eps$ is a finite subset of $S_1$ such that for all $ x \in S_1$, there exists $\tilde{x}\in S_\eps$ such that $\| x - \tilde{x} \| \leq \eps$.
  Since $S_1$ has dimension $2k$ and diameter $1$, we can find a set $S_\eps$ whose cardinality is at most  $\left(O\left(\frac{1}{\eps}\right)\right)^{2k}$.

  By a union bound, with probability $1 - e^{-\Omega(M)}|S_\eps|$, for all $\tilde{x} \in S_\eps$ there exists a subset of batches $J_{\tilde{x}} \subset [M]$ with $|J_{\tilde{x}}| \geq 0.95M$ such that 
  \begin{equation}
    \label{eqn: epsilon net srec paley}
    \frac{1}{b} \| A_{B_j} \tilde{x} \|^2 \geq \frac{C_1 t^2}{2} \; \forall \; j \in J_{\tilde{x}}
  \end{equation}
  Since $|S|_\eps \leq  e^{O(k \log \frac{1}{\eps})}$, if $M = \Omega \left( k \log \frac{1}{\eps} \right)$, the above statement holds with probability $1- e^{-\Omega(M)}$. 

  We now show that the statement of the proposition is true for all vectors in $S_1$.
  Since the proposition statement holds for an $\eps-$cover of $S_1$, we now only need to consider the effect of $A$ at a scale of $\eps$.

  Now consider the set
  $$S_2 = \{ x - \tilde{x} : x \in S_1, \tilde{x}\in S_\eps, \| x - \tilde{x} \| \leq \eps \}.$$ 

  Note that this a subset of all vectors in the span of $W$ that have norm at most $\eps$. That is, if
  \[
  S_3 = \{ Wz: z\in \R^{2k}, \| z \| \leq \eps \},
  \]
  we have $ S_2  \subseteq S_3 $.

  For a vector $v \in \R^n$, consider the random variable $$Z_i ( v ) = \bm{1} \left[ \inn{a_i}{v} \geq \frac{\sqrt{C_1}t }{2\sqrt{2}} \right]. $$
  Define the random process 
  $$ \Psi \left( a_1, a_2, \cdots , a_m \right) = \underset{ v\in S_2}\sup \frac{1}{m} \sum_{i=1}^{m} \bm{1} \left[ |\inn{a_i}{v}| \geq \frac{\sqrt{C_1}t }{2 \sqrt{2}} \right].$$

  By the bounded difference inequality, with probability $1 - 2e^{- C_2 \delta^{2}},$
  $$ \Psi(a_1, a_2, \cdots, a_m) \leq \E \left[\Psi ( a_1, a_2, \cdots, a_m) \right] + \frac{\delta}{\sqrt{m}}$$

  Since $S_2 \subseteq S_3$, we can bound the expectation of $\Psi$ by
  \begin{align}
    \E \left[ \Psi(a_1, \cdots, a_m) \right] & \leq \E\underset{ v\in S_3}\sup \frac{1}{m} \sum_{i=1}^{m} \bm{1} \left[ |\inn{a_i}{v}| \geq \frac{ \sqrt{C_1}t }{2\sqrt{2}} \right]  \\ 
    & \leq \E \underset{v \in S_3}\sup\sum_{i=1}^{m}  \frac{|\inn{a_i}{v}|}{m t \sqrt{C_1} / 2\sqrt{2} } \\
    &= \E \underset{v \in S_3}\sup\sum_{i=1}^{m}  \frac{2\sqrt{2} |\inn{a_i}{v}|}{m t \sqrt{C_1} }   \\
    & \leq \E \underset{v \in S_3}\sup \left|\sum_{i=1}^{m} 2\sqrt{2} \frac{ |\inn{a_i}{v}| - \E\left[|\inn{a}{v}\right|]}{m t \sqrt{C_1}}\right| +  \underset{v \in S_3}\sup \sum_{i=1}^{m} \frac{2\sqrt{2} \E\left[|\inn{a}{v}|\right]}{m t \sqrt{C_1}}
  \end{align}
  
  Since $a$ is isotropic and $v$ has norm at most $\eps$, by Jensen's inequality, we can bound the second term in the RHS by
  \begin{equation}\label{eqn: expectation sup srec}
  \E \underset{v \in S_3}\sup \sum_{i=1}^{m} \frac{2\sqrt{2} \E\left[|\inn{a}{v}|\right]}{m t \sqrt{C_1}} \lesssim \frac{\eps}{t\sqrt{C_1}}.
  \end{equation}
  
  To bound the first term in the RHS, we use the Gine-Zinn symmetrization inequality~\cite{gine1984some, mendelson2017aggregation, ledoux2013probability}
  \begin{align}
    &\E \underset{v \in S_3}\sup \left|\sum_{i=1}^m 2\sqrt{2} \frac{ |\inn{a_i}{v}| - \E\left[|\inn{a}{v}|\right]}{m t \sqrt{C_1}}\right| \lesssim \E \underset{v \in S_3}\sup \left|\sum_{i=1}^m \frac{ \xi_i \inn{a_i}{v} }{m t \sqrt{C_1}}\right|
  \end{align}
  where $\xi_i , i\in [m]$ are i.i.d $\pm 1$ Bernoulli variables.

  We can bound this by
   \begin{align}
     \E \underset{v \in S_3}\sup \left|\sum_{i=1}^m \frac{ \xi_i \inn{a_i}{v} }{m t \sqrt{C_1}}\right| & = \E_{\xi,A}\left[\sup_{v\in S_3 }\left|\frac{\xi^T Av}{m t \sqrt{C_1}}\right|\right],\\
     & = \E_{\xi,A}\left[\sup_{z:\|z\|\leq \eps}\left|\frac{\xi^T AWz}{m t \sqrt{C_1}}\right|\right]\\
     & \leq \E_{\xi,A}\left[\frac{\epsilon \|\xi^T AW\|}{m t \sqrt{C_1}}\right]\\
     & \leq \frac{\epsilon \sqrt{\E_{\xi,A}\|\xi^T AW\|^2} }{m t \sqrt{C_1}}\\
     & = \frac{\epsilon \sqrt{ \E_A \text{trace}(AWW^TA^T)} }{m t \sqrt{C_1}}\\
     & = \frac{\epsilon \sqrt{ 2km } }{m t \sqrt{C_1}} \lesssim \frac{\eps}{t}\sqrt{\frac{k}{m C_1}}
   \end{align}
   The third line follows from the Cauchy-Schwartz inequality, and the fourth line follows from Jensen's inequality.

   Since $m = Mb$, from the above inequality and Eqn~\eqref{eqn: expectation sup srec} we can now bound $\E \Psi$ as
   \begin{align}
     \E \left[ \Psi (a_1, \cdots , a_m) \right] & \lesssim \frac{\eps}{t}\sqrt{\frac{k}{M b C_1}} + \frac{\eps}{t\sqrt{C_1}}
   \end{align}

Substituting the above inequality into the bounded difference inequality, we have with probability at least $1 - e^{ -\Omega(\delta^2) }$,
\begin{align}
  \Psi(a_1, a_2, \cdots, a_m) \lesssim \frac{\eps}{t}\sqrt{\frac{k}{M b C_1}} + \frac{\eps}{t\sqrt{C_1}} + \frac{\delta}{\sqrt{Mb}} 
\end{align}

Setting $M = \Omega(k), \delta = O\left(\sqrt{\frac{M}{b}}\right), \eps = O\left( \frac{t}{b} \sqrt{C_1}\right)$, we can reduce the terms in the above inequality to 
\begin{align}
  \frac{\eps}{t}\sqrt{\frac{k}{M b C_1}} & \leq O \left( \frac{1}{b^{\frac{3}{2}}} \right), \\
  \frac{\eps}{t \sqrt{C_1}} &\leq O\left( \frac{1}{b}\right), \\
  \frac{\delta}{\sqrt{Mb}} &\leq O\left(\frac{1}{b}\right), 
\end{align}
Since $b>1$, the sum of these three terms is dominated by $O\left(\frac{1}{b}\right)$. 
From this, we can conclude that for small enough $\eps, \delta$, with probability $1 - e^{-\Omega\left(\frac{M}{b}\right)}$,
\begin{align}
  \Psi(a_1, a_2, \cdots, a_m) &\leq \frac{0.05}{b} \\
  \Rightarrow \underset{ v\in S_3}\sup \sum_{i=1}^{m} \bm{1} \left[ |\inn{a_i}{v}| \geq \frac{t\sqrt{C_1}}{2\sqrt{2}} \right] &\leq 0.05 M.
  \label{eqn: oscillation term srec}
\end{align}
This allows us to control the effect of $A$ at a scale of $\eps$.
It says that there at most $0.05M$ samples on which vectors with magnitude at most $\eps$ have a magnitude greater than $\frac{t\sqrt{C_1}}{2\sqrt{2}}$ after interacting with $A$.
This implies that there at least $0.95M$ batches in which all samples are well behaved.

Since we have control over an $\eps-$cover of $S_1$ as well as vectors at a scale of $\eps$ in $S_1$, we can now prove our result for all vectors in $S_1$.

For any $x \in S_1$, let $\tilde{x}\in S_\eps$ be the point in the $\eps-$cover which is closest to $x$. For a batch $B_j$, we can express $\| A_{B_j} x \|$ as 
\begin{align}
  \frac{1}{\sqrt{b}}\| A_{B_j} x\| &\geq \frac{1}{\sqrt{b}}\| A_{B_j} \tilde{x} \| - \frac{1}{\sqrt{b}}\| A_{B_j} (x - \tilde{x}) \|.
\end{align}

From \eqref{eqn: epsilon net srec paley},  there exists a subset of batches $J_{\tilde{x}} \subseteq [M]$ with $|J_{\tilde{x}}|\geq 0.95M$ such that 
\begin{align}
  \frac{1}{\sqrt{b}} \| A_{B_j} \tilde{x} \| \geq \frac{\sqrt{C_1} t}{\sqrt{2}} \; \forall \; j \in J_{\tilde{x}} .
\end{align}

From \eqref{eqn: oscillation term srec}, there exists a subset of batches $J_{x - \tilde{x}} \subseteq [M]$ with $|J_{x - \tilde{x}}| \geq 0.95 M$ such that for all $j\in J_{x - \tilde{x}}$,
\begin{align}
  |\inn{a_i}{ x - \tilde{x} }| &\leq \frac{\sqrt{C_1 }t}{2\sqrt{2}}\;\forall \; i\in B_j\\
  \Rightarrow \frac{1}{\sqrt{b}} \| A_{B_j} (x -\tilde{x}) \| &\leq \frac{\sqrt{C_1 }t}{2\sqrt{2}}, \\
  \Rightarrow -\frac{1}{\sqrt{b}} \| A_{B_j} (x -\tilde{x}) \|  &\geq -\frac{\sqrt{C_1 }t}{2\sqrt{2}}.
\end{align}

From the bounds on $\|A_{B_j} \tilde{x}\|$ and the bound on $\|A_{B_j} (x - \tilde{x} \|$, we can conclude that for all $x\in S_1$ there exist a subset of batches $J_x = J_{\tilde{x}} \cap J_{x - \tilde{x}}$ with cardinality at least $0.9M$ such that
\begin{align}
  \frac{1}{\sqrt{b}} \| A_{B_j} x \| &\geq \frac{\sqrt{C_1} t}{2 \sqrt{2}}, \; \forall\; j \in J_x.
\end{align}

This completes the proof, with $\gamma = \frac{\sqrt{C_1} t}{ 2\sqrt{2}} = \frac{t(1-t^2)}{C^2 2\sqrt{2}}$.
\end{proof}

\section{Proof of \Cref{lemma: multiplier process}}
\begin{lemma*}[\Cref{lemma: multiplier process}]
Consider the setting of~\Cref{lemma: srec} with measurements satisfying $y = AG(z^*) + \eta$. 
For any $t>0$ and noise variance $\sigma^2$, let the batch size $b$ and number of batches $M$ satisfy $b = \Theta(\frac{\sigma^2}{t^2})$ and $M=\Omega(kd\log n)$. Then with probability at least $1-e^{-\Omega(m)},$ for all $z\in \R^k$ there exists a set $J\subseteq \left[M\right]$ of cardinality at least $0.9M$ such that 
\[
\frac{1}{b} |\eta_{B_j}^T A_{B_j} (G(z) - G(z^*))| \leq t \| G(z) - G(z^*) \| \;, \forall j\in J.
\]

\end{lemma*}

\begin{proof}
  \Cref{prop: relu no of subspaces} shows that the set $S_G=\{G(z_1) - G(z_2): z_1, z_2 \in\R^k\}$ lies in the range of $e^{O(kd\log n)}$ different $2k-$dimensional subspaces.
  This trivially implies that for a fixed $z^*\in \R^k$, the set $\{G(z) - G(z^*): z \in\R^k\}$ also lies in the range of $e^{O(kd\log n)}$ different $2k-$dimensional subspaces.

  \Cref{prop: mom multiplier single subspace}  guarantees the result for a single subspace with probability $1-e^{-\Omega(M)}$. Since $M = \Omega(kd \log n)$ and the batch size is constant which depends on the noise variance $\sigma^2$ and $t^2$, the lemma follows from a union bound over the $e^{O(kd\log n)}$ subspaces.
\end{proof}

 \begin{proposition}
 \label{prop: mom multiplier single subspace}
 Consider a single $2k-$dimensional subspace given by $S=\{Wz: W\in\R^{n\times 2k}, W^T W=I_{2k}, z\in\R^{2k}\}$. 
 Let $A\in\R^{m\times n}$ be a matrix with i.i.d rows drawn from a distribution satisfying Assumption \eqref{assm:assumption} with constant $C$.
 If the batch size $b = \Theta\left(\frac{\sigma^2}{t^2}\right)$ and the number of batches satisfies $ M = \Omega\left( k \log \frac{1}{\eps} \right) $, with probability $1-e^{-\Omega(M)},$ for all $x \in S$, there exist a subset of batches $J_x \subseteq [M]$ with $|J_x|\geq 0.90M$ such that 
    \[
    \frac{1}{b} |\eta_{B_j}^T A_{B_j} x| \leq t \| x \| \;, \forall j\in J.
    \]
 \end{proposition}

\begin{proof}

  Since the bound we want to prove is homogeneous, it suffices to show it for all vectors in $S$ that have unit norm.
  Let $W \in \R^{n \times 2k}$ be the orthonormal matrix spanning $S$, and $S_1$ denote the set of unit norm vectors in its span. 
  That is, 
  \[
  S_1 = \{ Wz: z\in \R^{2k}, \| z \| = 1, W\in\R^{n\times 2k}, W^T W = I_{2k}\}.
  \]

  Consider the set $S_\eps$, which is a minimal $\eps-$covering of $S_1$. That is, for every $x \in S_1$, there exists $\tilde{x} \in S_\eps$ such that $\| \tilde{x} - x \| \leq \eps$.

  For a fixed $\tilde{x} \in S_\eps$, and $ t > 0$, by Chebyshev's inequality,
  \begin{align}
    \Pr\left[ \frac{1}{b} | \eta^T A_{B_j} \tilde{x}| \geq \frac{t}{2} \right] &\leq \frac{\sum_{i\in B_j} \left(\eta_i^2 \inn{a_i}{\tilde{x}}^2\right)}{b^2 t^2/4}\\
    &= \frac{b \sigma^2 \|\tilde{x}\|^2}{b^2 t^2 / 4} \\
    &= \frac{\sigma^{2} 4}{ b t^2 } \leq  \frac{1}{40},\label{eqn: chebyshev fixed vector multiplier}
  \end{align} if $ b \geq \frac{160 \sigma^{2}}{t^{2}} $.

  Define the indicator random variable
  $$ Y_i(x) = \bm{1} \left\{ \frac{1}{b} | \eta^T A_{B_i} x | \geq \frac{t}{2} \right\}.$$

  From Eqn~\eqref{eqn: chebyshev fixed vector multiplier} we have
  $$ \E \left[ Y_i(\tilde{x}) \right] \leq \frac{1}{40}.$$
  By concentration of Bernoulli variables, with probability $1-e^{-\Omega (M)}$,$$ \sum_{j=1}^M Y_i (\tilde{x}) \leq 2 \E \left[ Y_1(\tilde{x})\right] \leq \frac{1}{20}.$$

  This implies that for a fixed $\tilde{x}\in S_\eps,$ with probability $1-e^{-\Omega(M)},$ there exist a subset of batches $J_{\tilde{x}}\subseteq [M]$ with cardinality $0.95M$ such that  
  \begin{align}
    \frac{1}{b} | \eta^T A_{B_j} \tilde{x} | \leq \frac{t}{2}\; \forall \; j \in J_{\tilde{x}}.
  \end{align}

  Since the size of $S_\eps$ is at most $\left(O\left(\frac{1}{\eps}\right)\right)^{2k}$, we can union bound over all $\tilde{x}$ in $S_\eps$. 
  Hence, if $M = \Omega\left( k \log  \frac{1}{\eps}\right)$, then with probability $1-e^{-\Omega(M)}$,  for all $\tilde{x} \in S_\eps$, there exist a subset $J_{\tilde{x}}\subseteq [M]$ with cardinality $0.95M$ such that
  \begin{align}
  \label{eqn: multiplier upper bound eps net}
    \frac{1}{b} | \eta^T A_{B_j} \tilde{x} | \leq \frac{t}{2} \; \forall \; j\in J_{\tilde{x}}.
  \end{align}

  This shows that the multiplier component is well behaved on a large fraction of the batches for an $\eps-$cover of $S_1$. Now we need to extend the argument to all vectors in $S_1$.

  Now consider the set
  $$S_2 = \{ x - \tilde{x} : x \in S_1, \tilde{x}\in S_\eps, \| x - \tilde{x} \| \leq \eps \}.$$ 

  Note that this a subset of all vectors in the span of $W$ that have norm at most $\eps$. That is, if
  \[
  S_3 = \{ Wz: z\in \R^{2k}, \| z \| \leq \eps \},
  \]
  we have $ S_2  \subseteq S_3 $.

  For any $ v \in \R^n,$ define the random variable
  \begin{align}
    Z_j (v) = \bm{1}\left\{ | \eta_i a_i^T v | \geq \frac{t}{2}\right\}.
  \end{align}

  Now define the random process
  \begin{align}
    \Psi(a_1, \cdots, a_m) = \underset{v \in S_2}\sup \frac{1}{m} \sum_{i=1}^m Z_i(v)
  \end{align}

  Since $S_2\subseteq S_3$, we can bound $\E\left[\Psi\right]$ via
  \begin{align}
    &\E\left[\Psi\right] \leq \E\left[ \underset{v \in S_3}\sup\frac{1}{m} \sum_{i=1}^m Z_i(v) \right] \\
    &\leq \E\left[ \underset{v \in S_3}\sup \frac{1}{m} \sum_{i=1}^m \frac{ | \eta_i a_i^T v |}{t/2} \right] \\ 
    &\leq \E\left[ \underset{v \in S_3}\sup \left|\frac{1}{m} \sum_{i=1}^m \frac{ |\eta_i a_i^T v|  - \E|\eta_i a_i^T v |}{t/2}\right| \right] \nonumber\\
    & + \E\left[ \underset{v \in S_3}\sup \frac{1}{m} \sum_{i=1}^m \frac{  \E|\eta_i a_i^T v|}{t/2} \right] 
  \end{align}
  
  We can bound the term on the right by
  \begin{align}
  \label{eqn: expectation sup multiplier}
    \E\left[ \underset{v \in S_3}\sup \frac{1}{m} \sum_{i=1}^m \frac{  \E|\eta_i a_i^T v|}{t/2} \right] &\leq  \frac{\E\left[\underset{v \in S_3}{\sup}\| \eta_i \|_2 \;\;| \inn{a_i}{v}| \right]}{t/2}\\
    &\lesssim \frac{\sigma \eps}{t},
  \end{align}
  where we have used the Cauchy Schwartz inequality, followed by the fact that $\eta$ is independent noise and has variance $\sigma^2$, $a$ is isotropic, and $v\in S_3$ has norm at most $\eps$. 
  
  To bound the term on the left, we use the Gine-Zinn symmetrization inequality~\cite{gine1984some, mendelson2017aggregation, ledoux2013probability} 
  \begin{align}
    &\E\left[ \underset{v \in S_3}{\sup} \left|\frac{1}{m} \sum_{i=1}^m \frac{ |\eta_i a_i^T v|  - \E|\eta_i a_i^T v |}{t/2}\right| \right] \lesssim \E\left[ \underset{v \in S_3}{\sup} \left|\frac{1}{m} \sum_{i=1}^m \frac{ \xi_i\eta_i a_i^T v  }{t/2}\right| \right] 
    \end{align}
    where $\xi_i, i\in [m]$ are i.i.d $\pm$ Bernoulli random variables.
    
    Let $\xi\eta = (\xi_1 \eta_1, \xi_2 \eta_2, \cdots, \xi_m \eta_m)$ denote the the element wise product of the vectors $\xi=(\xi_1, \xi_2, \cdots, \xi_m)$ and $\eta = (\eta_1, \eta_2, \cdots, \eta_m)$.
  We can bound the above inequality by
   \begin{align}
     \E \underset{v \in S_3}\sup\left| \sum_{i=1}^m \frac{ \xi_i \eta_i \inn{a_i}{v} }{m t/2}\right| & = \E_{\xi,\eta, A}\left[\sup_{v\in S_3 }\left|\frac{(\xi\eta)^T Av}{m t/2 }\right|\right],\\
     & = \E_{\xi,\eta,A}\left[\sup_{z:\|z\|\leq \eps}\left|\frac{(\xi\eta)^T AWz}{m t/2}\right|\right]\\
     & \leq \E_{\xi,\eta,A}\left[\frac{\epsilon \|(\xi\eta)^T AW\|}{m t/2 }\right]\\
     & \leq \frac{\epsilon \sqrt{\E_{\xi,\eta,A}\|(\xi\eta)^T AW\|^2} }{m t/2}\\
     & = \frac{\epsilon \sigma \sqrt{ \E_A \text{trace}(AWW^TA^T)} }{m t/2}\\
     & = \frac{\epsilon \sigma \sqrt{ 2km } }{m t/2} \lesssim \frac{\eps \sigma}{t}\sqrt{\frac{k}{m}}
   \end{align}
   The third line follows from the Cauchy-Schwartz inequality, and the fourth line follows from Jensen's inequality, and the fifth line follows from the fact that $\xi\eta$ has i.i.d coordinates that are independent of $A$ and have variance $\sigma^2$.
    
    From the above inequality and \cref{eqn: expectation sup multiplier}, we get
    
    \begin{align}
        \E[\Psi(a_1, a_2, \cdots, a_m)]&\lesssim \frac{\sigma \eps }{t}\sqrt{\frac{k}{m}} + \frac{\sigma \eps}{t} \lesssim \frac{\sigma\eps}{t}
    \end{align}

  If we choose $\eps = c_1 \frac{t}{\sigma b}$ for a small enough constant $c_1$, then we can bound the expectation as
  \begin{align}
    \E\left[\Psi(a_1, \cdots, a_m) \right] &\leq  \frac{0.025}{b}
  \end{align}

  By the bounded differences inequality, with probability $1 - e^{-\Omega (\delta^2)}$,
  \begin{align}
    \Psi(a_1,\cdots,a_m) \leq \E\left[\Psi(a_1,\cdots, a_m)\right] + \frac{\delta}{\sqrt{m}}
  \end{align}

  Setting $\delta = 0.025 \sqrt{\frac{M}{b}} $, we get $\frac{\delta}{\sqrt{m}} = \frac{0.025}{\sqrt{Mb}}\sqrt{\frac{M}{b}} = \frac{0.025}{b}$. This gives
  \begin{align}
    \Psi(a_1,\cdots,a_m) &\leq \frac{0.025}{b} + \frac{0.025}{b} =\frac{0.05}{b}.
  \end{align}
  
  From which we conclude that
  \begin{align}
  \label{eqn: multiplier oscillation upper bound}
    \Rightarrow \underset{v \in S_2}\sup \sum_{i=1}^m \bm{1}\left\{ | \eta_i a_i^T v | \geq \frac{t}{2}\right\} &\leq \frac{0.05m}{b} = 0.05M.
  \end{align}

  Now consider any $ x \in S_1$. There exists $\tilde{x} \in S_\eps$ such that $\| \tilde{x} - x \| \leq \eps$ . From \cref{eqn: multiplier upper bound eps net} there exist a subset $J_{\tilde{x}}\subseteq [M]$ with cardinality $0.95M$ such that
  \begin{align}
    \frac{1}{b} | \eta_{B_j}^T A_{B_j} \tilde{x} | \leq \frac{t}{2} \; \forall\; j\in J_{\tilde{x}}.
  \end{align}
  
  Similarly, from \cref{eqn: multiplier oscillation upper bound}, there exists a subset $J_{x - \tilde{x}}\subseteq [M]$ with cardinality $0.95M$ such that for all $j\in J_{x - \tilde{x}},$ we have
  \begin{align}
    | \eta_i a_i^T (x-\tilde{x}) | &\leq \frac{t}{2}\;\forall\; i \in B_j,\\
    \Rightarrow \frac{1}{b}| \eta_{B_j}^T A_{B_j} (x-\tilde{x}) | &\leq \frac{t}{2}.
  \end{align}

  From the triangle inequality and a simple union bound, for all $x \in S_1$, there exists a subset $J_x = J_{\tilde{x}} \cap J_{x - \tilde{x}}$ with cardinality $0.9M$ such that 
  \begin{align}
    \frac{1}{b} | \eta_{B_j}^T A_{B_j} x | &\leq  \frac{1}{b} | \eta_{B_j}^T A_{B_j} (x - \tilde{x}) | + \frac{1}{b} |\eta_{B_j}^T A_{B_j} \tilde{x} |\\
    &\leq \frac{t}{2} + \frac{t}{2} = t 
  \end{align}

  This completes the proof.

\end{proof}

\section{Proof of \Cref{thm: mom tournaments}}
\label{sec:main_proof}

\begin{proof}
In \Cref{thm: mom tournaments}, we 
fix the batch size $b$ to be a suitable constant, specified in  \Cref{lemma: srec}, \Cref{lemma: multiplier process}. Then for $\eps \leq \frac{0.01}{b},$ the number of arbitrarily corrupted samples of $A$ and $y$ are at most $\frac{0.01}{b}bM = 0.01M$. This implies that there exist $0.99M$ batches with uncorrupted samples of $A,y$.
For the rest of the proof, consider only these uncorrupted batches, and ignore the corrupted batches.

For a batch $j$, define the following 
\begin{align}
    \mathbb{Q}_j(\wh{z},z^*) &:= 
    \frac{1}{b}\|A_{B_j} (G(\wh{z}) - G({z^*}))\|^2 ,\\
    \mathbb{M}_j({\wh{z}}) &:= \frac{2}{b}\eta_{B_j}^\top (A_{B_j} (G(\wh{z}) - G({z^*}))) \label{equ:processes}.
\end{align}
 it is easy to verify that $\ell_j ( \wh{z}) - \ell_j (z^*) = \mathbb{Q}_j(\wh{z},z^*) - \mathbb{M}_j({\wh{z}})$.
 The component $\mathbb{Q}_j(\wh{z},z^*)$ is commonly called the 
 quadratic component, and $\mathbb{M}_j({\wh{z}})$ is called the 
 multiplier component.
 
By~\Cref{lemma: obj min value}, the minimum value of the MOM objective is at most $4\sigma^2$ with high probability.
Since $\wh{z}$ minimizes the objective \cref{equ:MOM_minmax} to within additive $\tau$ of the optimum, it implies that the median batch satisfies 
\begin{align}\label{equ:tau}
   \mathbb{Q}_j(\wh{z},z^*) - \mathbb{M}_j({\wh{z}}) \leq 4 \sigma^2 + \tau.
\end{align}

Using~\Cref{lemma: srec},~\Cref{lemma: multiplier process} on the $0.99M$ batches that do not have corruptions, if the batch size is a large enough constant, we see that there exist $0.78M$ batches on which both the following inequalities hold
\begin{align}
\gamma^2 \|G(\wh{z}) - G(z^*)\|^2 \leq \mathbb{Q}_j(\wh{z},z^*) \; \hfill \text{ and } \hfill  \; -\sigma \|G(\wh{z}) - G(z^*)\| \leq -\mathbb{M}_j(\wh{z}). 
\end{align}

Putting the above two inequalities together, the median batch satisfies
\begin{align}
\gamma^2 \|G(\wh{z}) - G(z^*)\|^2 -\sigma \|G(\wh{z}) - G(z^*)\|  \leq 4\sigma^2  + \tau. \nonumber
\end{align}

Solving the quadratic inequality for $\norm{G(\wh{z}) - G(z^*)}$, we have
\begin{equation*}
\|G(\wh{z}) - G(z^*)\|^2 \lesssim \sigma^2 + \tau. \hfill \nonumber \qedhere
\end{equation*}

\end{proof}

\section{Experimental Setup}\label{app:experiments}
\subsection{MNIST dataset}
We first compare \Cref{alg:MOM_GAN} with the baseline ERM \cite{bora2017compressed} for heavy tailed dataset \emph{without} arbitrary corruptions on MNIST dataset \cite{lecun1998gradient}. We trained a DCGAN~\cite{radford2015unsupervised} to produce $64\times 64$ MNIST images.\footnote{Code was cloned from the following repository \url{https://github.com/pytorch/examples/tree/master/dcgan}.}
We choose the dimension of the latent space as $k = 100$, and the model has 5 layers.

Based on this generative model, the uncorrupted
compressed sensing model $P$ has heavy tailed measurement matrix and stochastic noise: $y = AG(z^*) + \eta$. We consider a Student's $t$ distribution (a typical example of heavy tails) -- the measurement matrix 
$A$ is generated from a Student's $t$ distribution  with  degrees of freedom 4, and  $\eta$ with degrees of freedom 3 with bounded variance $\sigma^2$. 
We vary the number of measurement  $m$ and obtain the reconstruction error $\|G(\wh{z}) - G(z^*)\|^2$ for \Cref{alg:MOM_GAN} and ERM, where $G(z^*)$ is the ground truth image.
Each curve in \Cref{fig:Curve} demonstrates the averaged reconstruction error for 50 trials. 
In \Cref{fig:Curve}, \Cref{alg:MOM_GAN} and ERM both have decreasing reconstruction error per pixel with increasing number of measurement. In particular, \Cref{alg:MOM_GAN}
obtains significantly smaller reconstruction error  comparing with the baseline ERM.

\subsection{CelebA-HQ dataset}

We continue the  study  of empirical performance of our algorithm on real image datasets with higher quality. 
We generate high quality RGB images with size $256\times 256$ from CelebA-HQ\footnote{Code was cloned from the following repository: \url{https://github.com/facebookresearch/pytorch_GAN_zoo}.}. Hence the dimension of each image is $256\times256\times3 = 196608$. In all of our experiments, we fix the 
 dimension of the latent space as
$k = 512$, and
 train a DCGAN on this dataset to obtain a generative model $G$.

We first compare our algorithm with the baseline ERM \cite{bora2017compressed}  for heavy tailed dataset without arbitrary corruptions, and then deal with 
the situation of outliers.

\paragraph{Heavy tailed samples.}
In this experiment, we deal with
the \textit{uncorrupted} compressed sensing model $P$, which has heavy tailed measurement matrix and
stochastic noise: $y = AG(z^*) + \eta$. 
We also use a Student’s t distribution for $A$ and $\eta$  – the measurement matrix $A$ is generated from a Student’s t distribution with degrees of freedom 4, and stochastic noise $\eta$ with degrees of
freedom 3 with a bounded variance. 

We obtain the reconstruction error $\norm{G(\widehat{z}) - G(z^*))}$ vs. the number of measurement $m$ for our algorithm and ERM, where $z^*$ is the ground truth. In \Cref{fig:Curve_post}, each curve is an average of 20 trials.
For heavy tailed $y$ and
$A$ without any corruption,
both methods are consistent, 
and have decaying reconstruction error with increasing sample size.
Our method obtains significantly smaller reconstruction error,
and shows competitive results over the baseline ERM for heavy tailed data set, even without any arbitrary outliers.

\subsection{Hyperparameter selection}
When using the Adam~\cite{kingma2014adam} optimizer, we varied the learning rate over $[0.1,0.05, 0.01, 0.005]$ for our 
algorithm and baselines.
When using the Yellowfin~\cite{zhang2017yellowfin} optimizer, we varied our learning rates over $[10^{-4}, 5\cdot 10^{-5}, 10^{-5}, 5\cdot 10^{-6}, 10^{-6}]$.
We selected the best learning rate based
on fresh measurements that were not used for optimization.

\section{Background}
\label{sec:background}

\begin{theorem}[Ledoux-Talagrand Contraction Inequality]
For a compact set $\mathcal{T}$, let $x_1,\cdots,x_m$ be i.i.d vectors whose real valued components are indexed by $\mathcal{T}$, i.e., $x_i = (x_{i,s})_{s\in\mathcal{T}}$. Let $\phi:\R\rightarrow\R$ be a 1-Lipschitz function such that $\phi(0)=0$. Let $\epsilon_1,\cdots,\epsilon_m$ be independent Rademacher random variables. Then
\[
\E\left[\sup_{s\in\mathcal{T}}\left|\sum_{i=1}^m\epsilon_i \phi(x_{i,s})\right|\right] \leq 
2 \E\left[\sup_{s\in\mathcal{T}}\left|\sum_{i=1}^m\epsilon_i x_{i,s}\right|\right].
\]
\end{theorem}

\begin{theorem}[Talagrand's Inequality for Bounded Empirical Processes]
For a compact set $\mathcal{T}$, let $x_1,\cdots,x_m$ be i.i.d vectors whose real valued components are indexed by $\mathcal{T}$, i.e., $x_i = (x_{i,s})_{s\in\mathcal{T}}$. Assume that $\E x_{i,s}=0$ and $|x_{i,s}|\leq b$ for all $s\in\mathcal{T}$. Let $Z=\sup_{s\in\mathcal{T}}\left|\frac{1}{m}\sum_{i=1}^m x_{i,s}\right|$. Let $\sigma^2=\sup_{s\in\mathcal{T}} \E x^2_{s}$ and $\nu = 2b\E Z +\sigma^2$. Then
\[
\Pr\left[Z \geq \E Z + t\right] \leq C_1\exp\left({-\frac{C_2 mt^2}{\nu+bt}}\right).
\]
where $C_1, C_2$ are absolute constants.
\end{theorem}

\end{document}